\documentclass{article}

\usepackage[utf8]{inputenc}
\usepackage[a4paper, left=1.2in, right=1.2in, top=1.2in, bottom=1.2in]{geometry}
\usepackage{authblk}

\usepackage[numbers,sort&compress]{natbib}
\usepackage{microtype}
\usepackage{graphicx}
\usepackage{booktabs}
\usepackage[hidelinks]{hyperref}
\usepackage{amsmath}
\usepackage{amssymb}
\usepackage{mathtools}
\usepackage{amsthm}
\usepackage{mathrsfs}
\usepackage{thmtools}
\usepackage{xspace}
\usepackage[parfill]{parskip}
\usepackage{algorithm}
\usepackage{algorithmic}
\usepackage{wrapfig}
\usepackage{multirow}
\usepackage{enumitem}
\usepackage{bibunits}
\usepackage{makecell}
\usepackage{subfig}
\theoremstyle{plain}
\newtheorem{theorem}{Theorem}[section]
\newtheorem{proposition}[theorem]{Proposition}
\newtheorem{lemma}[theorem]{Lemma}

\theoremstyle{definition}
\newtheorem{definition}[theorem]{Definition}

\theoremstyle{remark}

\usepackage[textsize=tiny]{todonotes}

\usepackage{amsmath,amsfonts,bm}
\usepackage{amsthm}

\usepackage{hyperref}
\usepackage{url}
\usepackage{booktabs}
\usepackage{amsfonts}
\usepackage{xcolor}         
\usepackage{xspace}
\usepackage{amsmath}
\usepackage{amssymb}
\usepackage{mathtools}
\usepackage{amsthm}
\usepackage{mathrsfs}
\usepackage{thmtools}
\usepackage{algorithm}
\usepackage{algorithmic}
\usepackage{graphicx}
\usepackage{wrapfig}
\usepackage{multirow}
\usepackage{enumitem}
\usepackage{bibunits}
\usepackage{makecell}

\makeatletter
\DeclareRobustCommand\onedot{\futurelet\@let@token\@onedot}
\def\@onedot{\ifx\@let@token.\else.\null\fi\xspace}

\def\eg{\emph{e.g}\onedot} 
\def\ie{\emph{i.e}\onedot}

\makeatother

\DeclarePairedDelimiter\norm{\lVert}{\rVert}
\DeclarePairedDelimiter\innorm{\langle}{\rangle}

\DeclareMathOperator{\St}{\mathcal{S}}
\DeclareMathOperator{\A}{\mathcal{A}}

\DeclareMathOperator{\Z}{\mathcal{Z}}

\newcommand{\Pc}{\mathcal{P}}

\newcommand{\Uc}{\mathcal{U}}

\definecolor{cornflower}{RGB}{100,149,237}

\def\eqref#1{equation~\ref{#1}}

\def\1{\bm{1}}

\DeclareMathAlphabet{\mathsfit}{\encodingdefault}{\sfdefault}{m}{sl}
\SetMathAlphabet{\mathsfit}{bold}{\encodingdefault}{\sfdefault}{bx}{n}

\newcommand{\R}{\mathbb{R}}

\DeclareMathOperator*{\argmin}{arg\,min}

\title{Bring Your Own (Non-Robust) Algorithm to Solve Robust MDPs by Estimating The Worst Kernel}

\author[*1]{Kaixin Wang}
\author[*1]{Uri Gadot}
\author[1]{Navdeep Kumar}
\author[1]{Kfir Levy}
\author[1,2]{Shie Mannor}

\affil[1]{Technion}
\affil[2]{NVIDIA Research}

\begin{document}
\def\thefootnote{*}\footnotetext{Authors contributed equally to this work}

\maketitle
\begin{abstract}
Robust Markov Decision Processes (RMDPs) provide a framework for sequential decision-making that is robust to perturbations on the transition kernel.
However, current RMDP methods are often limited to small-scale problems, hindering their use in high-dimensional domains. 
To bridge this gap, we present \textbf{EWoK}, 
a novel online approach to solve RMDP that \textbf{E}stimates the \textbf{Wo}rst transition \textbf{K}ernel to learn robust policies.
Unlike previous works that regularize the policy or value updates, EWoK achieves robustness by simulating the worst scenarios for the agent while retaining complete flexibility in the learning process.
Notably, EWoK can be applied on top of any off-the-shelf {\em non-robust} RL algorithm, enabling easy scaling to high-dimensional domains.
Our experiments, spanning from simple Cartpole to high-dimensional DeepMind Control Suite environments, demonstrate the effectiveness and applicability of the EWoK paradigm as a practical method for learning robust policies.
\end{abstract}

\section{Introduction}

In reinforcement learning (RL), we are concerned with learning good policies for sequential decision-making problems modeled as Markov Decision Processes (MDPs)~\citep{Puterman1994MarkovDP,Sutton1998}.
MDPs assume that the transition model of the environment is fixed across training and testing, but this is often violated in practical applications.
For example, when deploying a simulator-trained robot in reality, a notable challenge is the substantial disparity between the simulated environment and the intricate complexities of the real world, leading to potential subpar performance upon deployment.
Such a mismatch may significantly degrade the performance of the trained policy (in testing).
To deal with this issue, the robust MDP (RMDP) framework has been introduced in~\citep{iyengar2005robust,nilim2005robust,wiesemann2013robust}, aiming to learn policies that are robust to any perturbation of the transition model provided it lies within an uncertainty set.

\begin{figure*}[t]
\centering
\includegraphics[width=0.9\linewidth]{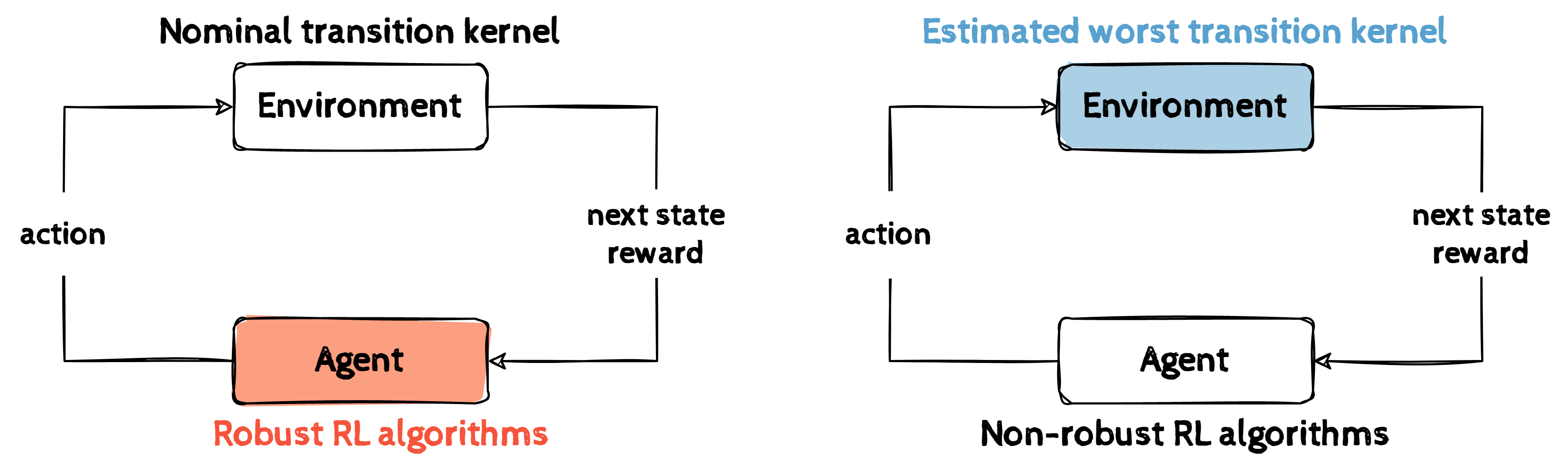}
\caption{The agent-environment interaction loop during training. \textbf{Left}: Existing methods typically regularize how an agent updates its policy to improve robustness. \textbf{Right}: Our work estimates a worst transition kernel, so the agent essentially learns its policy under the worst scenarios and can use any non-robust RL algorithm.}
\label{fig:intro}
\end{figure*}

Existing works on learning robust policies in RMDPs often suffer from poor scalability(to high-dimensional domains).
Specifically, model-based methods that solve RMDPs~\citep{wiesemann2013robust,ho2021partial,behzadian2021fast,derman2021twice,grand2021scalable,kumar2022efficient} require access to the nominal transition probability, making it difficult to scale beyond tabular settings.
While some recent works~\citep{wang2022robust,wang2022policy,kumar2022efficient,PgLprmdp} introduce model-free methods that add regularization to the learning process, the effectiveness of their methods is not validated in high-dimensional environments.
In addition, these methods are based on particular RL algorithms (\eg, policy gradient, Q learning), limiting their general applicability.
We defer a more detailed discussion on related works to Section~\ref{sec:related}.

In this work, we tackle the problem of learning robust policies in RMDPs from an alternative direction.
As shown in Figure~\ref{fig:intro}, unlike previous works that explicitly regularize the learning process, we propose to approximately sample next states from an \textbf{E}stimated \textbf{Wo}rst transition \textbf{K}ernel (EWoK) while leaving the RL part untouched.
In RMDPs, a worst transition kernel is one within the uncertainty set that leads to the minimal possible return (see Definition~\ref{dfn:adv_kernel_def}).
Intuitively, EWoK aims to situate the agent in the worst scenarios for learning policies robust to perturbations.
It can be applied on top of any (deep) RL algorithm, offering good scalability to high-dimensional domains.

Specifically, EWoK builds upon our theoretical insights into the relationship between a worst transition kernel and the nominal one.
Our characterization of the worst kernel for a KL-regularized uncertainty set concludes that it essentially modifies the next-state transition probability of the nominal kernel, discouraging the transitions to states with higher values while encouraging transitions to lower-value states.
Using this connection, we are able to sample the next states such that they are approximately distributed according to the worst transition probability.
We establish convergence of the estimated worst kernel to the true worst kernel and present a practical algorithm suitable for high-dimensional domains.

To verify the effectiveness of our method, we conduct experiments on multiple environments ranging from small-scale classic control tasks to high-dimensional
DeepMind Control tasks~\citep{tunyasuvunakool2020}.
The agent is trained on the nominal environment and tested in environments with perturbed transitions.
Since our method is agnostic to the underlying RL algorithm,
we showcase the applicability of our method on top of different non-robust algorithms.
Experiment results demonstrate that with our method, the learned policy suffers from less performance degradation when the transition kernel is perturbed, even when the perturbation is situated within an uncertainty set that is either coupled or non-KL based.

In summary, our paper makes the following contributions:
\begin{itemize}[leftmargin=15pt]
    \item To learn robust policies in RMDPs, we propose to approximately simulate the ``worst" transition kernel, instead of regularizing the learning process. This opens up a new paradigm for learning robust policies in RMDPs.
    \item We theoretically characterize the ``worst" kernel for KL uncertainty sets, which is amenable to approximate simulation for environments with large state spaces.
    \item Our method is not tied to a particular RL algorithm and can be easily integrated with any deep RL method.
    This flexibility translates to the good scalability of our method in complex high-dimensional domains.
    To the best of our knowledge, our work is the first that enjoys such flexibility among related works in RMDPs.
\end{itemize}

\section{Preliminaries}

\textit{Notations.} For a finite set $\Z$, we write the probability simplex over it as $\Delta_{\Z}$.
Given two real functions $f,g:\Z\to\R$, their inner product is $\innorm{f,g} = \sum_{z\in\Z}f(z)g(z)$.
For distributions $P,Q$, we denote the Kullback–Leibler (KL) divergence of $P$ from $Q$ by $D_\mathrm{KL}(P\,\lVert\,Q)$.

\subsection{Markov Decision Processes}

A Markov decision process (MDP)~\citep{Sutton1998, Puterman1994MarkovDP} is a tuple $(\St,\A,P,R,\gamma,\mu)$, where $\St$ and $\A$ are the state space and the action space respectively,
$P:\St\times\A\to \Delta_{\St}$ is the transition kernel, $R:\St\times\A\to \R$ is the reward function,
$\gamma\in [0,1)$ is the discount factor,
and $\mu\in\Delta_{\St}$ is the initial state distribution.
A stationary policy $\pi:\St\to \Delta_{\A}$ maps a state to a probability distribution over $\A$.
We use $P(\cdot|s,a)\in\Delta_{\St}$ to denote the probabilities of transiting to the next state when the agent takes action $a$ at state $s$.
For a policy $\pi$, we denote the expected reward and transition by:
\begin{align*}
    R^\pi(s) &= \sum_{a\in\A}\pi(a|s)R(s,a), \\
    P^\pi(s'|s) &= \sum_{a\in\A}\pi(a|s)P(s'|s,a), \quad \forall s, s' \in \St
\end{align*}
The value function $v^\pi:\St\to\mathbb{R}$ maps a state to the expected cumulative reward when the agent starts from that state and follows policy $\pi$, \ie,
\begin{equation*}
    v^\pi(s) = \mathbb{E}_{\pi, P} \left[\sum_{t=0}^\infty \gamma^t R(s_t,a_t) \Biggm\vert s_0=s \right].
\end{equation*}
It is known that $v^\pi$ is the unique fixed point of the Bellman operator $T^\pi v:=R^\pi + \gamma P^\pi v$~\citep{Puterman1994MarkovDP}.
The agent's objective is to obtain a  policy $\pi^*$ that maximizes the discounted return
\begin{equation*}
        J^\pi = \mathbb{E}_{\mu,\pi,P}\left[\sum_{t=0}^\infty \gamma^t R(s_t,a_t) \right] = \langle \mu, v^\pi\rangle.
\end{equation*}

\subsection{Robust Markov Decision Processes}

In MDPs, the system dynamics $P$ is usually assumed to be constant over time. However, in real-life scenarios, it is subject to perturbations, which may significantly impact the performance in deployment~\citep{mannor2007bias}.
Robust MDPs (RMDPs) provide a theoretical framework for taking such uncertainty into consideration, by taking $P$ as not fixed but chosen adversarially from an uncertainty set $\Pc$~\citep{iyengar2005robust,nilim2005robust}.
Since we may consider different dynamics $P$ in the RMDPs context, in the following, we will use subscript $P$ to make the dependency explicit.
The objective in RMDPs is to obtain a policy $\pi^*_\Pc$ that maximizes the robust return
\begin{equation*}
    J^\pi_{\Pc} = \min_{{P\in\Pc}} J^\pi_P.
\end{equation*}
However, 
this problem is NP-hard for general uncertainty sets while an optimal policy can be non-stationary~\citep{wiesemann2013robust}. To make RMDPs tractable, we need to make some assumptions about the uncertainty set.

\subsection{Rectangular uncertainty set}
\label{sec:pre-uncertainty}

One commonly used assumption to enable tractability for RMDPs is rectangularity.
Specifically, we assume that the uncertainty set $\Pc$ can be factorized over states-actions:
\begin{equation}
    \Pc = \underset{(s,a)\in(\St\times\A)}{\times}\Pc_{sa}, \tag{\texttt{sa}-rectangularity}
\end{equation}
where $\Pc_{sa}\subseteq\Delta_{\St}$.
In other words, the uncertainty in one state-action pair is independent of that in another state-action pair.

Under this assumption, RMDPs admit a deterministic optimal policy as in standard MDPs~\citep{iyengar2005robust,nilim2005robust}.
The rectangularity assumption also allows the robust value function to be well-defined:
\begin{equation*}
    v^\pi_\Pc = \min_{P\in \Pc} v^\pi_P, \quad\textrm{and}\quad
    v^*_\Pc = \max_{\pi} v^\pi_\Pc.
\end{equation*}
In addition, $v^\pi_\Pc$ and $v^*_\Pc$ are the unique fixed points of the robust Bellman operator $T_{\Pc}^\pi$ and the optimal robust Bellman operator $T_{\Pc}^*$ respectively, (which are) defined as
\begin{equation*}
    T_{\Pc}^\pi v (s) =\min_{P\in\Pc} T_P^\pi v (s)
    \,\textrm{and}\,
    T_{\Pc}^* v(s) = \max_{\pi} T_{\Pc}^\pi v(s).
\end{equation*}

To model perturbations on the environment dynamics, the (rectangular) uncertainty set is often constructed (to be centered) around a nominal kernel $\bar{P}$.
Since we want to measure the divergence between probability distributions, it is natural to use KL divergence~\citep{panaganti2022sample,xu2023improved,shi2022distributionally}, \ie,
\begin{equation*}
    \Pc_{sa} =\{P_{sa}\in\Delta_{\St}\mid D_\mathrm{KL}(P_{sa} \,\lVert\,\bar{P}_{sa})\le\beta_{sa}\}. 
\end{equation*}
Here $P_{sa}$ is a shorthand for $P(\cdot|s,a)$ and $\beta_{sa}$ is the uncertainty radius that controls the level of perturbation.

\section{Method}
\label{sec:method}


\begin{figure}[t]
\centering
\includegraphics[width=0.7\linewidth]{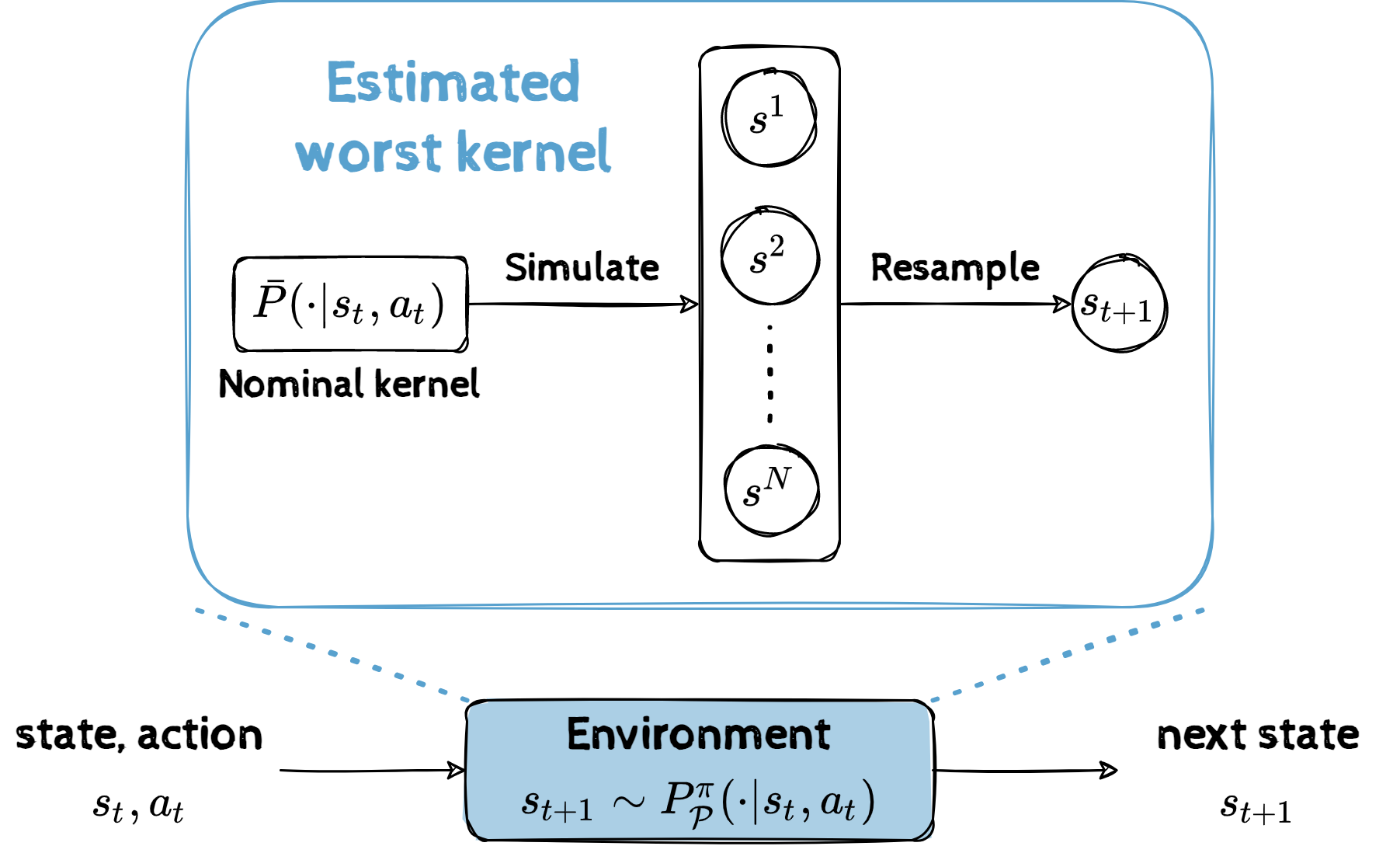}
\caption{An illustration of how next states are sampled in the estimated worst kernel.}
\label{fig:method}
\end{figure}

As introduced earlier, our work proposes to learn robust policies by approximately simulating a {\em worst transition kernel}, (which is) defined as the one within the uncertainty set that achieves minimal robust return. We formalize it below.
\begin{definition}
\label{dfn:adv_kernel_def}
For an uncertainty set $\Pc$ and a policy $\pi$, a worst kernel is defined as
\begin{equation*}
    P^\pi_\mathcal{\Pc}
    \in \argmin_{P\in\,\Pc} J^\pi_P.
\end{equation*}
\end{definition}
Training policies under this worst kernel will give us a robust policy with respect to the uncertainty set.
Note that $P^\pi_\mathcal{\Pc}$ itself is nothing more than a regular transition kernel.
Learning a policy under $P^\pi_\mathcal{\Pc}$ is no different from the standard MDP setting and we can adopt any non-robust RL algorithms to solve it.
The challenge is how to approximately simulate this worst kernel $P^\pi_{\Pc}$.
For a general uncertainty set $\mathcal{P}$, it requires an additional minimization process to find a worst kernel and it is also unclear how we can parameterize and learn $P^\pi_{\Pc}$ effectively.

To tackle this challenge, we characterize the connection between the nominal transition kernel and a worst one.
With such a connection, we are able to obtain the next states that are approximately distributed according to $P^\pi_{\Pc}(\cdot|s,a)$, by properly resampling the next states from the nominal kernel (Figure~\ref{fig:method}).
Formally, the following theorem describes this connection. All proofs are deferred to Appendix~\ref{app:proof}.
\begin{restatable}{theorem}{klworst}
\label{thm:kl}
    For a KL uncertainty set $\Pc$ and a policy $\pi$, a worst kernel is related to the nominal kernel through:
    \begin{equation*}
        P^\pi_{\Pc}(s'|s,a) = \bar{P}^\pi(s'|s,a)e^{-\delta^\pi(s')},
    \end{equation*}
    where $\delta^\pi$ is of the form
    \begin{equation}
        \delta^\pi(s') = \frac{v_{\Pc}^\pi(s')-\omega_{sa}}{\kappa_{sa}},
    \label{eqn:delta_pi}
    \end{equation}
    and satisfies
    \begin{equation}
        \label{eq:delta}
        \begin{aligned}
        &\sum_{s'\in\St} \bar{P}^\pi(s'|s,a)e^{-\delta^\pi(s')}=1,\\
        &\sum_{s'\in\St}\bar{P}^\pi(s'|s,a)e^{-\delta^\pi(s')}(-\delta^\pi(s'))=\beta_{sa}.
        \end{aligned}
    \end{equation}
\end{restatable}

Here, $\omega_{sa}$ and $\kappa_{sa}$ are implicitly defined by Eqn.~(\ref{eq:delta}). While they do not have closed forms, we can view $\omega_{sa}$ as a threshold, encouraging transitions to states with robust values lower than $\omega_{sa}$ (\ie, $\delta^\pi(s') < 0$), and discouraging transitions to states with higher robust values.
$\kappa_{sa}$ works as a temperature parameter to control how much we discourage/encourage transitions to states with high/low robust value.
More specifically, the following proposition explicates the relationship between $\omega_{sa}$ and $\kappa_{sa}$ and the uncertainty radius $\beta_{sa}$.
\begin{restatable}{proposition}{meanvar}\label{prop:meanvar}
$\omega_{sa}$, $\kappa_{sa}$ and $\beta_{sa}$ satisfy
\begin{equation*}
    \omega_{sa} = \langle P^\pi_\Pc(\cdot|s,a),v^\pi_\Pc\rangle + \beta_{sa}\kappa_{sa},
\end{equation*}
\end{restatable}

Based on theoretical results, we arrive at a method to approximately simulate a worst kernel.
As illustrated in Figure~\ref{fig:method}, we first draw a batch of states from the nominal kernel $\bar{P}(\cdot|s,a)$ , then resample the next state with probability proportional to $e^{-\delta^\pi(s')}$.
This way, next states will be approximately distributed according to $P^\pi_{\Pc}(\cdot|s,a)$.
In practice, we approximate $\delta^\pi(s')$ by
\begin{equation*}
    \hat{\delta}^\pi(s') = \frac{v(s')-\frac{1}{N}\sum_{i=1}^{N}v(s^i)}{\kappa},
\end{equation*}
where $v$ is the robust value function approximated with neural networks, and $\kappa$ is a hyperparameter controlling the robustness level (for all $s,a$, we assume $\kappa_{sa} = \kappa$).
We implement the threshold $\omega$ as the average value, a choice supported by the following proposition.
\begin{restatable}{proposition}{meanbound}
\label{prop:meanbound}
$\omega_{sa}$ is bounded as follows,
\begin{equation*}
    \langle P^\pi_\Pc(\cdot|s,a),v^\pi_\Pc\rangle\leq \omega_{sa} \leq \langle \bar{P}^\pi(\cdot|s,a),v^\pi_\Pc\rangle.
\end{equation*}
\end{restatable}
Since the $N$ next states are sampled from the nominal kernel, we can approximate an upper bound of $\omega_{sa}$ and use it as a proxy to compute $\hat{\delta}^\pi$.
Putting it together, we summarize our method in Algorithm~\ref{alg:aka}.

\begin{algorithm}[tb]
\caption{EWoK - Learning robust policy by Estimating Worst Kernel}
\label{alg:aka}
\textbf{Input}: sample size $N$, robustness parameter $\kappa$ \\
\textbf{Initialize}: initial state $s_0$, policy $\pi$ and value function $v$, data buffer
\begin{algorithmic}[1]
\FOR{$t=0,1,2,\cdots$}
    \STATE Play action $a_t\sim \pi(\cdot|s_t)$.
    \STATE Simulate next state $s^i \sim \bar{P}(\cdot|s_t,a_t), i=1,\cdots,N$, with the nominal environment dynamic.
    \STATE Choose $s_{t+1} = s^i$ with probability proportional to $e^{- \frac{v(s^i)-\frac{1}{N}\sum_{i=1}^{N}v(s^i)}{\kappa}}$.
    \STATE Add $(s_t,a_t,s_{t+1})$ to the data buffer.
    \STATE Train $\pi$ and $v$ with data from the buffer using any non-robust RL method.
\ENDFOR
\end{algorithmic}
     \end{algorithm}

\textbf{Convergence.}
The core of our method is the estimation of a worst transition kernel.
In practice, however, we do not have the true robust value function as in Eqn.~(\ref{eqn:delta_pi}).
We start with a randomly initialized value function and expect it to gradually converge to the robust value over training.
Here, we give some theoretical analysis on the convergence of this process.
Let $P_{n}$ denote the estimated worst transition kernel at iteration $n$ and $v_{P_n}^\pi$ denote the (non-robust) value function for the transition kernel $P_n$.
We are interested in the convergence of the following updates:
\begin{equation}
    P_{n+1}(s'|s,a) = \bar{P}(s'|s,a) e^{-\frac{v^\pi_{P_n}(s')-\omega_n}{\kappa_n}}.
\label{eqn:converge}
\end{equation}
$\omega_n$ and $\kappa_n$ are associated with the worst-case transition kernel when the target function is $v^\pi_{P_n}$.
For clarity, we omit their subscript $sa$ (even though they depend on $\beta_{sa}$).
The following theorem shows that the value converges to the robust value and the estimated kernel converges to a worst kernel $P^\pi_\Pc$.
\begin{restatable}{theorem}{kernelconv}
\label{thm:kernelconv}
For the updating process in Eqn.~(\ref{eqn:converge}), we have
\begin{equation*}
    \norm{v^\pi_{P_{n}} - v_{\Pc}^{\pi}}_{\infty} \leq  \gamma^n \norm{v^\pi_{\bar{P}} - v_{\Pc}^{\pi}}_{\infty}. 
\end{equation*}
\end{restatable}
Note that using the robust value function, a worst kernel $P^\pi_\Pc$ can be computed as $P^\pi(\cdot|s,a) = \bar{P}_{\Pc}(\cdot|s,a)e^{-\frac{v^\pi_\Pc-\omega_{sa}}{\kappa_{sa}}}$ as in Theorem \ref{thm:kl}.
This worst kernel (or the samples from it) can be used with any non-robust RL method for policy improvement as described in Figure \ref{fig:method}.

\subsection{A Cliff Walking Example}

\begin{wrapfigure}{r}{0.5\textwidth}
\centering
\includegraphics[width=.9\linewidth]{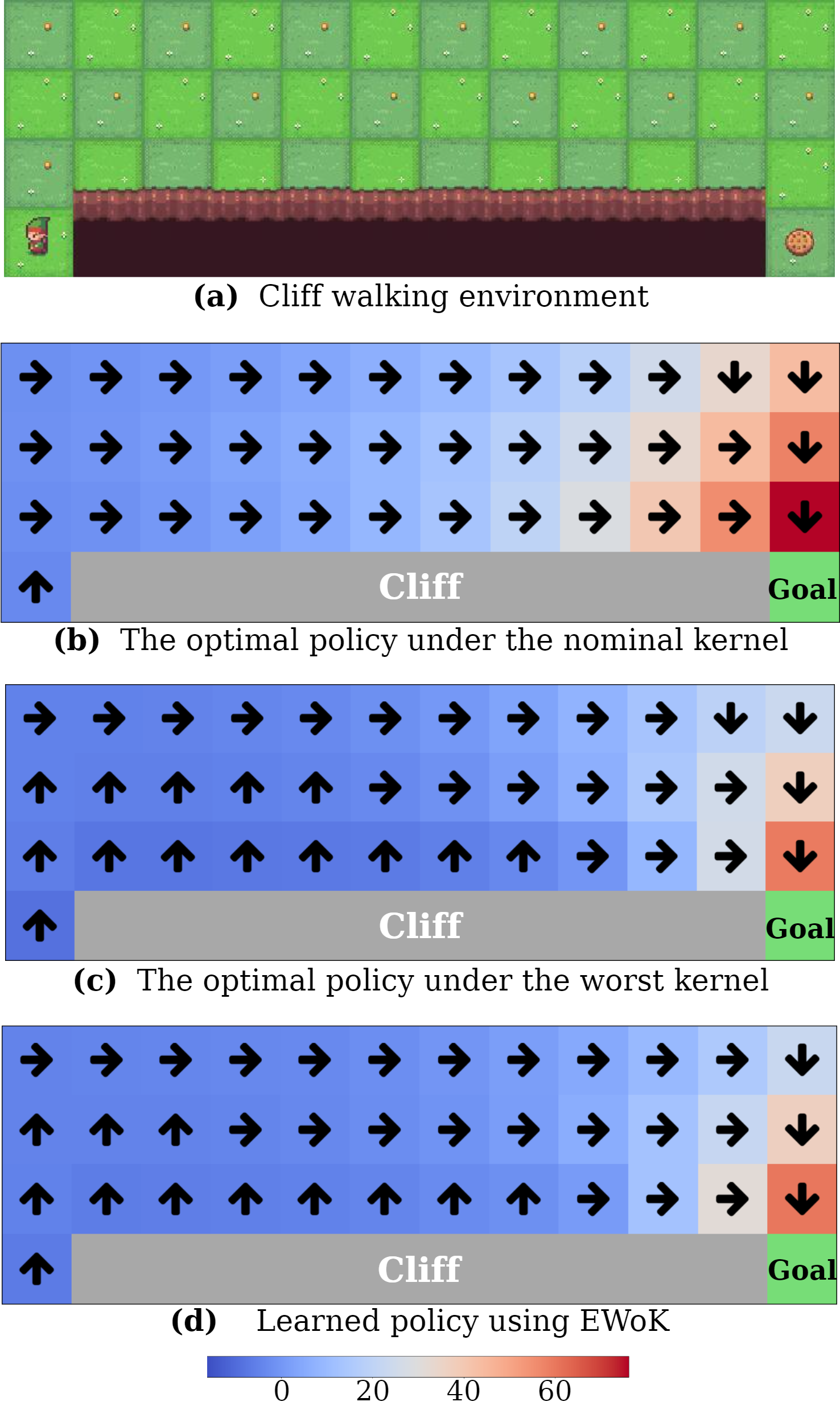}
\vspace{-10pt}
\caption{
Cliff-Walking environment and experiment results. In the bottom 3 plots, the color indicates the learned value and the arrows indicate the actions under the policy.}
\label{fig:cliff}
\end{wrapfigure}


To check if our algorithm can reliably learn the robust value function, we test it on a toy task based on OpenAI's Cliff Walking environment~\citep{brockman2016openai}. 
As shown in Figure~\ref{fig:cliff}(a), the agent must reach the goal state as quickly as possible by moving in 4 cardinal directions on a grid.
If the agent falls off a cliff then it will suffer a penalty and be teleported to the initial state.
The nominal transition kernel is modified to incorporate stochasticity; specifically, with a small probability the agent may go to other adjacent states instead of the state indicated by its action.
The uncertainty is implemented by varying such probabilities within some range.
Please refer to Appendix~\ref{app:env} for details.

Since this environment is tabular, we can obtain the ground-truth optimal robust value and policy by solving an optimization problem (see Appendix~\ref{app:worst-env}).
Then, for both the nominal and the worst transition kernels, we compute their optimal value functions and policies using value iteration.
The results are plotted in Figure~\ref{fig:cliff}(b) and (c).
We can see that even though there is a small chance that the agent would fall off the cliff, the optimal policy of the nominal kernel still advises the agent to move right.
However, under the worst kernel, the optimal policy tends to avoid walking adjacent to the cliff.

Next, we apply EWoK on top of Q-learning~\citep{watkins1992q} to learn the optimal robust value and policy, only using samples from the nominal kernel.
As Figure~\ref{fig:cliff}(d) shows, EWoK learns a policy closely resembling the optimal robust one, advising the agent to stay away from the cliff.
This preliminary experiment acts as a proof-of-concept, demonstrating the efficacy of the proposed method.
In the subsequent section, we will conduct a more thorough evaluation of EWoK in environments of greater complexity.

\section{Experiments}

\begin{figure*}[h]
\centering
\includegraphics[width=0.6\linewidth]{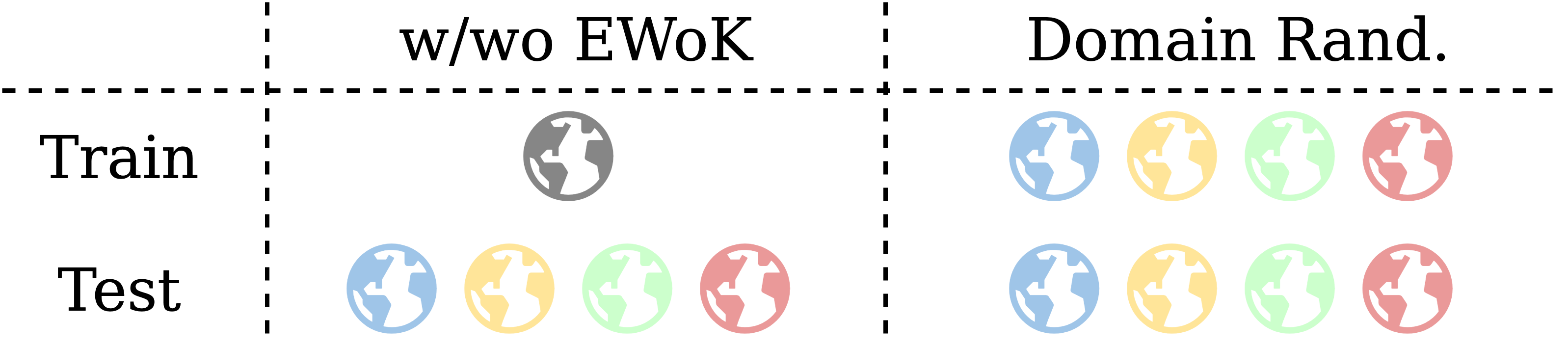}
\caption{An illustration of the experimental setting.
Grey earth denotes the unperturbed (nominal) environment while colored earths denote perturbed environments.}
\label{fig:exp_setup}
\end{figure*}


\subsection{Setting}
\label{sec:setting}

\begin{figure*}[t]
\centering
\includegraphics[width=1.\linewidth]{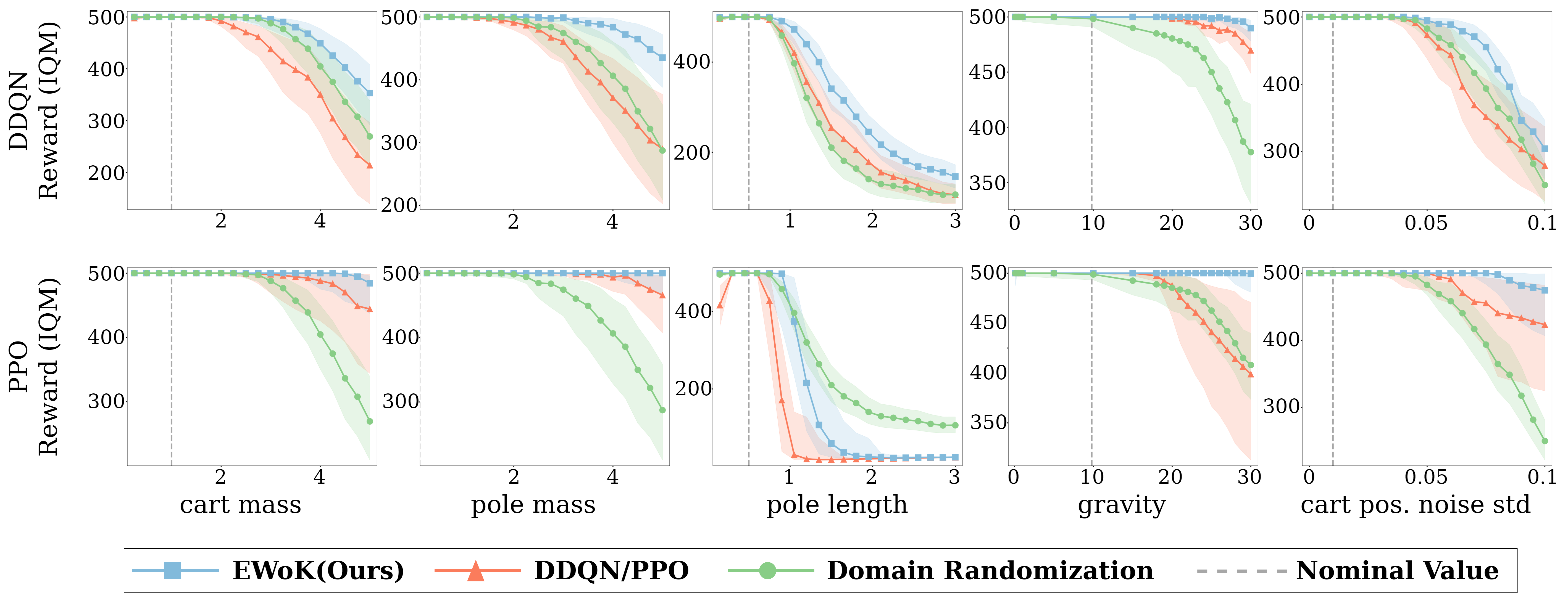}
\caption{Evaluation results on Cartpole with noise and environment parameters perturbations for both DDQN and PPO algorithms.}
\label{fig:cartpole}
\end{figure*}

To evaluate the effectiveness of our method in learning robust policies, we conduct experiments that train the agent online under nominal dynamics and test its performance under perturbed dynamics.
We consider two high-dimensional domains including both discrete and continuous control tasks, to demonstrate that our algorithm can be ``plugged and played'' with any RL method.
Specifically, we experiment on Cartpole - a classic control environment from OpenAI's Gym~\citep{brockman2016openai}
and 3 continuous control tasks (Walker-run, Walker-stand, Walker-walk) from DeepMind Control Suite~\citep{tunyasuvunakool2020}.
For baseline RL algorithms, we use Double DQN~\citep{ddqn} and PPO \citep{schulman2017proximal} for Cartpole, and SAC~\citep{haarnoja2018soft} and TD3~\citep{fujimoto2018addressing} for continuous control environments.
It is worth mentioning that these realistic perturbations do not adhere to the KL uncertainty assumptions.
Experimenting with realistic uncertainties (\eg, coupled or non-KL-based) would demonstrate the general applicability of EWoK, beyond the scope of its theoretical motivation.

As existing methods in RMDPs literature do not scale well (see discussions in Section~\ref{sec:related}), we can not clearly compare ``apples-to-apples''. Therefore, we consider another commonly used robust RL approach as a reference: domain randomization~\citep{tobin2017domain}, and conduct the same set of experiments.
Domain randomization trains the agent under diverse scenarios by perturbing the parameter of interest during training, such that the trained agent can be robust to similar perturbations during testing. It is worth noting that domain randomization has an edge on our method, since it has access to different perturbed parameters during training, while our method is oblivious to those parameters. Figure~\ref{fig:exp_setup} illustrates the differences.

To obtain stable results, we run each experiment with multiple random seeds, and report the interquartile mean (IQM) and 95\% stratified bootstrap confidence intervals (CIs) as recommended by~\citep{agarwal2021deep}.
More details about environments, implementations, training, and evaluation can be found in Appendix~\ref{app:exp-details}.

\subsection{Noise perturbation}
\label{subsection:added_noise}

\begin{figure*}[t]
\centering
\includegraphics[width=0.97\linewidth]{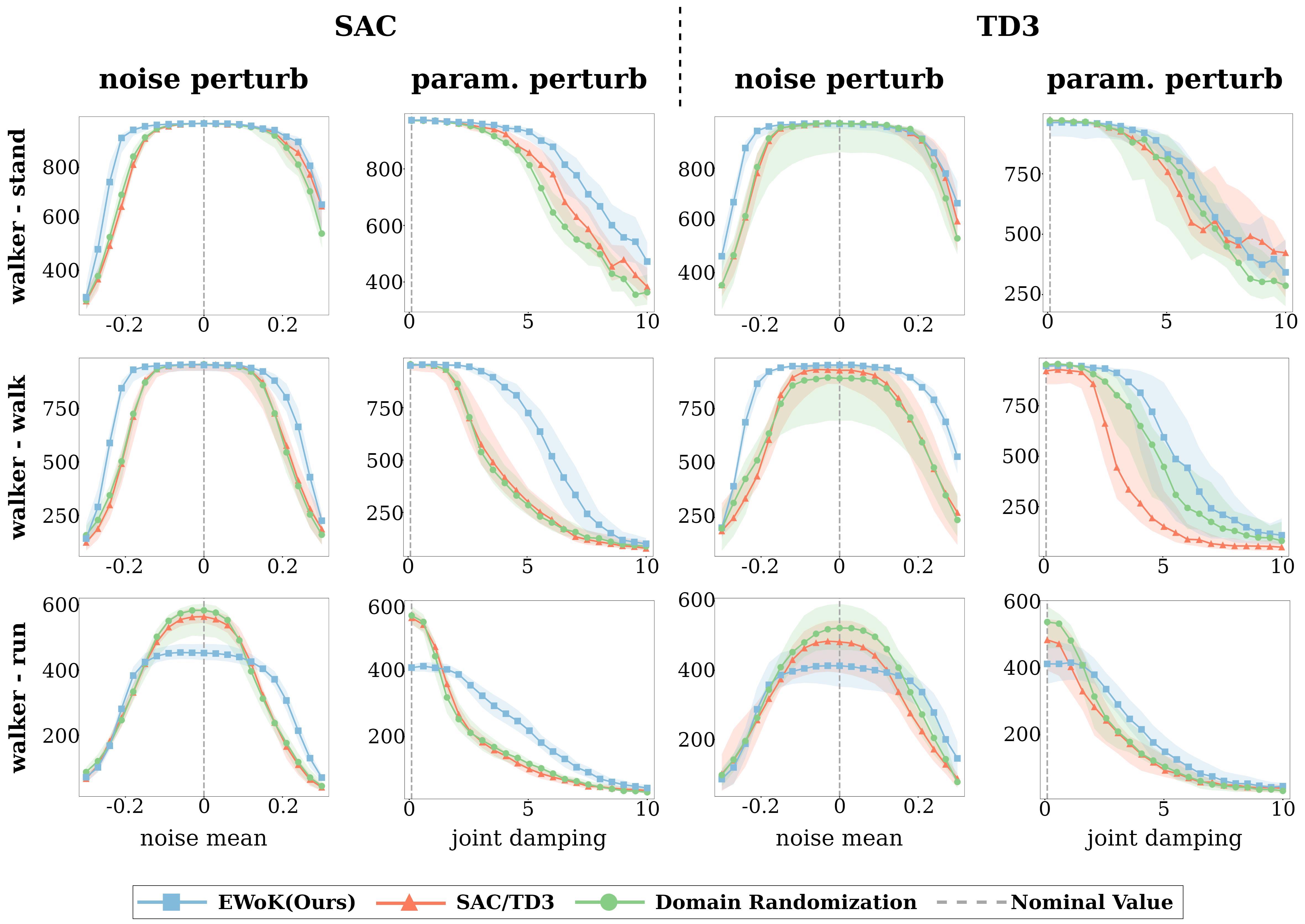}
\caption{Evaluation results on DeepMind Control environments with noise and environment parameter (joint damping) perturbations for both SAC and TD3 algorithms.}
\label{fig:dm_control}
\end{figure*}

In this subsection, we evaluate our method in scenarios where perturbations on the transition dynamics are implemented as noise perturbations.
Specifically, we consider stochastic nominal kernels in which the stochasticity is controlled by some (observation or action) noise.
The agent is trained under a fixed noise (\ie, the nominal kernel) and tested with varying noises (\ie, perturbed kernels).

On Cartpole, we implement the stochasticity by adding Gaussian noise to the state after applying the original deterministic dynamics of the environment, \ie, $\tilde{s}_{t+1} = s_{t+1} + \epsilon$ where $\epsilon \sim \mathcal{N}(0,\sigma)$.
Then, $\tilde{s}_{t+1}$ is considered as the next state output from the stochastic nominal kernel.
The noise scale $\sigma$ is fixed during training and varies during testing.
The agent's test performance across different perturbed values is depicted in the rightmost plot in Figure~\ref{fig:cartpole}. When the noise scale deviates from the nominal value, EWoK achieves better performance than the baseline non-robust RL algorithm and the domain randomization mechanism.

Next, we evaluate our method on continuous control tasks in the DeepMind Control Suite.
The stochasticity is implemented by adding Gaussian noise to the action since directly adding noise to the state might lead to an invalid physical state.
During testing, we perturb the mean of the Gaussian noise.
Figure~\ref{fig:dm_control} shows the agent's performance across different perturbed values.
We can see that EWoK suffers from less performance degradation as the noise mean deviates from zero (the nominal value), clearly outperforming the baseline non-robust RL algorithm.
In the walker-run task, EWoK achieves lower reward under the nominal dynamic but performs better under perturbed ones, which indicates a trade-off between the performance under the nominal kernel and robustness under perturbations.

\subsection{Perturbing environment parameters}
\label{subsection:env_params}

To further validate the effectiveness of our method, we consider a more realistic scenario where some physical/logical parameters in the environment (\eg, pole length in Cartpole) are perturbed.
Similarly, the agent is trained with a fixed parameter, and tested under perturbed parameters.

For Cartpole, we perturb cart mass, pole mass, pole length, and gravity.
Figure~\ref{fig:cartpole} summarizes the testing results of the agents trained under the nominal dynamics.
Again, EWoK achieves better performance than the baseline non-robust RL algorithm and the domain randomization technique when the environment parameters deviate from the nominal value.
It is noteworthy that for every environmental parameter (\eg, pole mass), we train a separate domain randomization agent that undergoes perturbations only on that parameter during training.
In comparison, EWoK is trained only once and then tested on different parameters.

For DeepMind control tasks, we implement the perturbations on the environment parameters using the Real-World Reinforcement Learning Suite~\citep{dulac2020empirical}. Specifically, we perturb joint damping, thigh length, and torso length. For all of the results, please refer to Appendix~\ref{app:all-results}.
As shown in Figure~\ref{fig:dm_control}, EWoK generally works better than the baseline under model mismatch, improving the robustness of the learned policy.
Similar to our observations in the previous section, the walker-run task emphasizes the inherent trade-off of solving RMDPs: optimizing the worst-case scenario can lead to suboptimal performance under the nominal model.

\subsection{Ablation studies}
\label{subsection:ablation}

In this subsection, we conduct ablation experiments to investigate the effects of our hyperparameters on the performance. Recall that $\kappa$ controls the skewness of the distribution for resampling, while $N$ controls the number of next-state samples.
Intuitively, when we decrease $\kappa$, we are essentially considering a higher level of robustness.
If $\kappa$ is very small, then with a high probability the environment dynamic will transit to the ``worst'' state (\ie, one with the lowest value).
In addition, by increasing $N$ we effectively improve our empirical estimation of the nominal kernel's next state distribution, which should improve the worst kernel estimation.

We experiment on the DeepMind Control tasks under noise perturbation setting, using different $\kappa$ and $N$ when we train the agent.
For clarity, we plot the performance difference between our method and the baseline instead of the absolute performance and defer the original results with CIs (shaded areas) to Appendix~\ref{app:all-results}.
Figure~\ref{subfig:ablation-kappa} shows the results of changing the values of $\kappa$.
In the walker domain, decreasing $\kappa$ makes our algorithm perform better in perturbed environments, which aligns with our expectations.
Figure~\ref{subfig:ablation-N} shows the results of changing the values of $N$.
We can see that a small sample size will result in limited performance gain compared to the baseline, but increasing the sample size may not bring monotonic improvements.
In addition, more samples will incur longer simulation time in each environment step.
In our experiments, we observed minimal impact on walk-clock time, due to fast simulation. In practical scenarios where sampling next states could be slow, however, we need to take this factor into consideration. Nonetheless, we believe should not significantly increase simulation time to a prohibitive extent.

It is worth mentioning that the influence of $\kappa$ depends on the environment.
Decreasing it too much can lead to too conservative policies and may not always work well.
In addition, we observe the robustness-performance trade-off in the walker-run task once again.
While large $\kappa$ achieves high performance under the nominal kernel, it significantly under-performs when the kernel is perturbed.

\begin{figure*}[t]
\centering
\subfloat[{different $\kappa$ values}]{
  \includegraphics[width=0.45\linewidth]{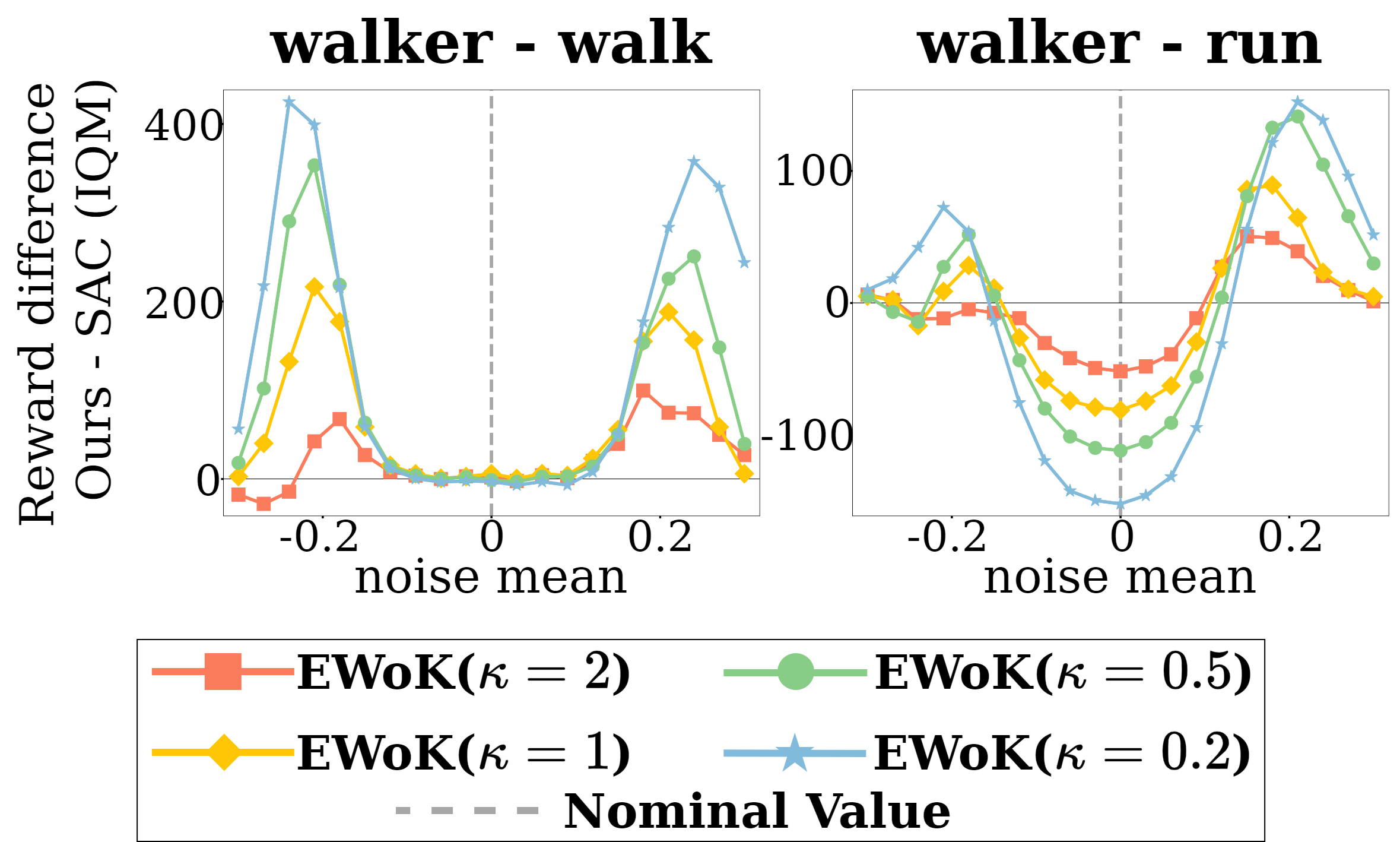}
  \label{subfig:ablation-kappa}}
\hfill
\subfloat[different $N$ values \label{subfig:ablation-N}]{
\includegraphics[width=0.45\linewidth]{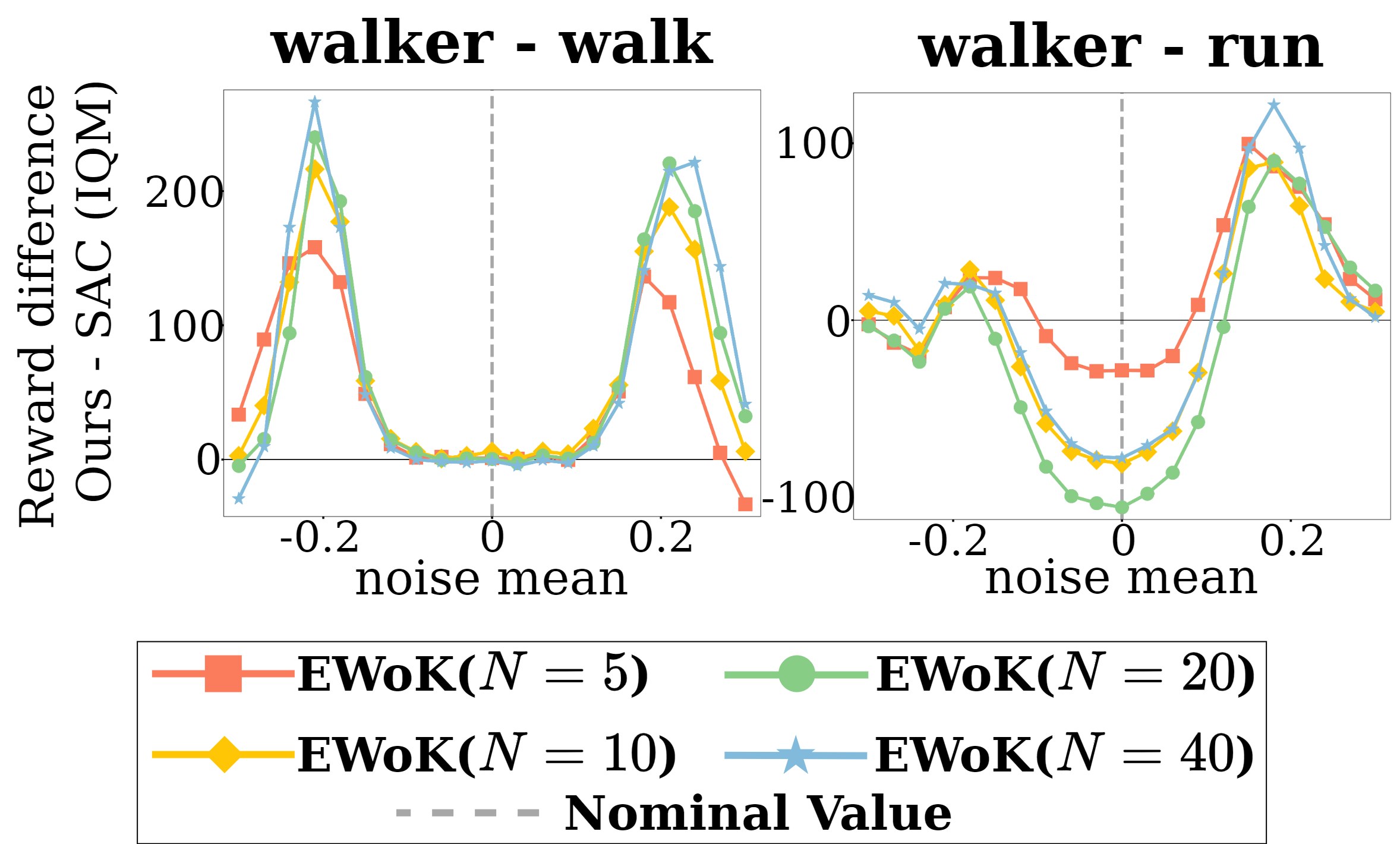}}
\caption{Evaluation results on DeepMind Control tasks with noise perturbations for different $\kappa / N$. The y-axis represents the performance difference between our method and the baseline}
\end{figure*}


\section{Related works}
\label{sec:related}


Early works in RMDPs lay the theoretical foundations for solving RMDPs with robust dynamic programming~\citep{wiesemann2013robust,iyengar2005robust,nilim2005robust,Kaufman2013RobustMP,bagnell2001solving}.
Recent works attempt to reduce the time complexity for certain uncertainty sets, such as $L_1$ uncertainty~\citep{ho2018fast,ho2021partial} and KL uncertainty~\citep{grand2021scalable}.
However, they require full knowledge of the nominal model.
 

One line of work aims to design methods that can be applied in the online robust RL setting where we do not have full knowledge about the transition model.
\citet{derman2021twice} define new regularized robust Bellman operators that suggest a possible online sample-based method.
However, the contraction of the Bellman operators implicitly assumes that the state space can not be very large.
On regularizing the learning process, \citet{kumar2022efficient, PgLprmdp} introduce Q-learning and policy gradient methods for $L_p$ uncertainty sets, but do not experimentally evaluate their methods with experiments.
Another type of uncertainty is the R-contamination, for which previous works have derived a robust Q-learning algorithm~\citep{Rcontamination} and a regularized policy gradient algorithm~\citep{wang2022policy}.
R-contamination uncertainty assumes that the adversary can take the agent to any state, which is too conservative in practice. 
In addition, all of those methods are tied to a particular type of RL algorithm.
Our work, however, aims to tackle the problem from a different perspective by approximating a worst kernel and can adopt any non-robust RL algorithm to learn an optimal robust policy.
A recent work~\citep{wang2023policy} has shown that the worst kernel can be computed using gradient descent, but their method takes more iterations to converge.


Outside of RMDP literature, perturbing the training environment was previously discussed in unsupervised environment design~\citep{dennis2020emergent,jiang2021replay}, domain randomization~\citep{peng2018sim,tobin2017domain}, robust adversarial RL~\citep{pinto2017robust,rigter2022rambo} and risk aversion~\citep{greenberg2022efficient,pan2019risk}.
However, their focus on robustness differs from our perspective.
They either assume access to environment parameters or aim for better generalization.
Our method is theoretically driven as a solution to RMDPs.


Our work is also closely related to~\citep{PgLprmdp}, which characterizes the worst kernel for $L_p$ uncertainty set.
Different from their work, we propose to approximately simulate this worst kernel, opening a new paradigm for learning robust policies in RMDPs.
The work of~\citep{zhou2023natural} ran parallel to ours. They employ a sample method to establish a new manageable uncertainty set, enabling the computation of a robust Bellman operator through both value-based and policy-based methods. In contrast, our approach involves estimating the worst kernel through sampling from the nominal one to address the problem.


\section{Conclusions and discussion}
\label{sec:conclusion}

In this paper we introduce an approach that tackles the RMDP problem from a new perspective, by approximately simulating a worst transition kernel while leaving the RL part untouched.
The highlight of our method is that it can be applied on top of any existing non-robust deep RL algorithms to learn robust policies, exhibiting attractive scalability to high-dimensional domains.
We believe this new perspective will offer some insights for future works on RMDPs.

One limitation of our work is that we require the ability to sample next states from the transition model multiple times.
In future work, we will study how to combine our method with a learned transition model (\ie, world models ~\citep{DBLP:conf/iclr/HafnerLB020, DBLP:conf/nips/HaS18}) where the challenge of next-state sampling is mitigated. We also believe that using EWoK for model-based or offline setups might lower the effect of compounding error issue~\citep{asadi2018lipschitz}.



\bibliography{biblio}
\bibliographystyle{bib_style}


\newpage
\appendix
\onecolumn

\section{Proof}
\label{app:proof}

\subsection{Proof of Theorem~\ref{thm:kl}}

Recall that the worst values are defined as 
\[ P^\pi_\Pc \in \argmin_{P\in \Pc}J^\pi_{P}\]
for any general uncertainty set $\Pc$. Further, for $sa$-rectangular uncertainty set $\Pc = \times_{s\in\St, a\in\A}\Pc_{sa}$, the robust value function exists, that is, the following is well defined~\citep{nilim2005robust,iyengar2005robust}
\[v^\pi_{\Pc} = \min_{P\in \Pc}v^\pi_{P}.\]
This implies,
\[v^\pi_{\Pc}  = \Bigm(I-\gamma (P^\pi_{\Pc})^\pi\Bigm)^{-1} R^\pi\]
is the fixed point of robust Bellman operator $\mathcal{T}^\pi_{\Pc}$\citep{nilim2005robust,iyengar2005robust}, defined as
\[\mathcal{T}^\pi_{\Pc}v := \min_{P\in \Pc}\mathcal{T}^\pi_{P}v.\]

\begin{proposition}\label{app:wortkrnBellmn} The worst values can be computed from the robust value function. That is
    \[\argmin_{P\in \Pc}\mathcal{T}^\pi_{P}v^\pi_{\Pc} \subseteq \argmin_{P\in \Pc}v^\pi_{P} \subseteq \argmin_{P\in \Pc}J^\pi_{P} .\]
\end{proposition}
\begin{proof} 
Let \[P^* \in \argmin_{P\in \Pc}\mathcal{T}^\pi_{P}v^\pi_{\Pc}.\]
Now, from the fixed point of robust Bellman operator, we have
\begin{align*}
v^\pi_{\Pc} = &\mathcal{T}^\pi_{\Pc}v^\pi_{\Pc} = \min_{P\in \Pc}\mathcal{T}^\pi_{P}v^\pi_{\Pc},\\
 &=\mathcal{T}^\pi_{P^*}v^\pi_{\Pc},\qquad \text{(by construction)},\\
 &=R^\pi + \gamma (P^*)^\pi v^\pi_{\Pc},\qquad \text{(by definition)}.\\
\end{align*}
The above implies,
\[ v^\pi_{\Pc} =  \Bigm(I-\gamma (P^*)^\pi\Bigm)^{-1}R^\pi.\]
This implies, 
\[P^*\in \argmin_{P\in \Pc}v^\pi_{P}.\]
The last inclusion is trivial, that is, every minimizer of value function is a minimizer of robust return.
\end{proof}

\klworst*
\begin{proof}
Recall Definition~\ref{dfn:adv_kernel_def}
\begin{equation*}
    P_{\Pc}^\pi \in \argmin_{P\in\Pc} J^\pi_P.
\end{equation*}
From Proposition \ref{app:wortkrnBellmn}, for \texttt{sa}-rectangular uncertainty set $\Pc$, a worst kernel can be computed using robust value function as 
\[P^\pi_\Pc \in \argmin_{P\in \Pc}\mathcal{T}^\pi_{P}v^\pi_{\Pc}.\]

Recall, our KL-constrained uncertainty $\Pc$ is defined as

\[\Pc :=  \{P \mid P \in (\Delta_{\St})^{\St\times\A}, D_{KL}(\bar{P}_{s,a}, P_{sa}) \leq \beta_{sa}, \forall s,a\}.\]

where $D_{KL}$ is KL norm that is defined as 
\[D_{KL}(P,Q) = \sum_{s}P(s)\log\left(\frac{P(s)}{Q(s)}\right).\]

Using Proposition \ref{app:wortkrnBellmn} and definition of uncertainty set $\Pc$, the worst kernel can be extracted as 
\[P^\pi_\Pc(\cdot|s,a) \in \argmin_{D_{KL}(p,P_0(\cdot|s,a)) \leq \beta_{sa}, \sum_{s}p(s)=1, p\succeq 0}\quad \langle p, v^\pi_\Uc\rangle.\]
Using the Lemma~\ref{app:Langrange}, we get the desired solution.
\end{proof}

\begin{lemma} \label{app:Langrange} For $q \in \Delta_{\St},v \in \mathbb{R}^{\St}, \beta \geq 0$, a solution to 
    \[\min_{p\ln(\frac{p}{q}) \leq \beta, 1^Tp =1,  p\succeq 0} \langle p,v\rangle.  \]
    is given by 
    \[  p = qe^{-\frac{v-\omega}{\lambda}},\]
where
\[ p\log(\frac{p}{q}) = \Bigm \langle qe^{-\frac{v-\omega}{\lambda}}, \frac{v-\omega}{\lambda}\Bigm\rangle = -\beta \]
and 
\[\sum_{s}q(s)e^{-\frac{v(s)-\omega}{\lambda}}=1. \]
\end{lemma}
\begin{proof} We have the following optimization problem,
\begin{align*}
    \min_{p\ln(\frac{p}{q}) \leq \beta, 1^Tp =1,  p\succeq 0} \langle p,v\rangle.  
\end{align*}
We ignore the constraint $ p\succeq 0$ for the moment (as we see later, this constrained is automatically satisfied), and focus on 
    \begin{align*}
    \min_{p\ln(\frac{p}{q}) \leq \beta, 1^Tp =1} \langle p,v\rangle.
\end{align*}

We define Lagrange multiplier as 
\[L(p,\lambda,\mu) = \langle p,v\rangle + \lambda \Bigm(p\ln(\frac{p}{q}) - \beta\Bigm) + \mu \Bigm(1^Tp -1\Bigm). \]
We now put the stationarity condition:
\begin{align*}
    &\frac{\partial L}{\partial p} = v + \lambda \Bigm(\ln(\frac{p}{q}) + 1 \Bigm) + \mu 1 = 0\\
\implies & p = qe^{-1}e^{-\frac{v+\mu}{\lambda}}.\\
\end{align*}
With appropriate change of variable $\mu \to \omega$, we have
\[  p = qe^{-\frac{v-\omega}{\lambda}}.\]
We have to find the constants $\omega$ and $\lambda$, using the constraints
\[ p\log(\frac{p}{q}) = \Bigm \langle qe^{-\frac{v-\omega}{\lambda}}, \frac{v-\omega}{\lambda}\Bigm\rangle = -\beta\]
and 
\[ \sum_{s} p(s) = \sum_{s}q(s)e^{-\frac{v(s)-\omega}{\lambda}} = 1.\]
We further note that the constraint $1\geq p(s)\geq 0$ is automatically satisfied as 
\[ p(s) = q(s)e^{-\frac{v(s)-\omega}{\lambda}} \geq 0 \]
and $\sum_{s}p(s) =1$, ensures $p(s)\leq 1 \quad \forall s$.
\end{proof}

\subsection{Proof of Proposition~\ref{prop:meanvar} and ~\ref{prop:meanbound}}

\begin{proposition}\label{rs:MeanUpperBound} $\omega_{sa}$ can be upper-bounded as follows,
    \[\omega_{sa} \leq \langle \bar{P}(\cdot|s,a), v^\pi_\Pc\rangle,\qquad \forall s\in\St,a\in\A.\]
\end{proposition}
\begin{proof}
From the constraint in Theorem~\ref{thm:kl}, we have
\begin{align*}
&\sum_{s'}\bar{P}(s'|s,a)e^{-\frac{v^\pi_\Pc(s')-\omega_{sa}}{\kappa_{sa}}} = 1 \\
\implies & e^{-\sum_{s'}\bar{P}(s'|s,a)\frac{v^\pi_\Pc(s')-\omega_{sa}}{\kappa_{sa}}} \leq 1 \qquad \text{(using Jenson's inequality)} \\
\implies & e^{-\sum_{s'}\bar{P}(s'|s,a)\frac{v^\pi_\Pc(s')}{\kappa_{sa}}}e^{\frac{\omega_{sa}}{\kappa_{sa}}} \leq 1 \\
\implies & \frac{\omega_{sa}}{\kappa_{sa}} \leq \sum_{s'}\bar{P}(s'|s,a)\frac{v^\pi_\Pc(s')}{\kappa_{sa}}\\
\implies & \omega_{sa} \leq \sum_{s'}\bar{P}(s'|s,a)v^\pi_\Pc(s'). \qedhere
\end{align*}
\end{proof}

\meanvar*
\begin{proof}
From the constraint in Theorem~\ref{thm:kl}, we have
\begin{align*}
& \sum_{s'} \bar{P}^\pi(s'|s,a)e^{-\delta^\pi(s')}(-\delta^\pi(s'))=\beta_{sa} \\
\implies & \sum_{s'} P^\pi_\Pc(s'|s,a)\frac{v^\pi_\Pc(s')-\omega_{sa}}{\kappa_{sa}} =\beta_{sa} \\
\implies&\sum_{s'}P^\pi_\Pc(s'|s,a)v^\pi_\Pc(s') =-\beta_{sa}\kappa_{sa} + \omega_{sa}. \qedhere
\end{align*}    
\end{proof}

\meanbound*
\begin{proof}
    The lower bound is direct from Proposition~\ref{prop:meanvar}, as $\beta$ and $\kappa$ are positive quantities by definition. The upper bound comes from Proposition \ref{rs:MeanUpperBound}.
\end{proof}


\subsection{Proof of Theorem~\ref{thm:kernelconv}}

Given a policy $\pi$, let $P_{n+1}$ be the updated kernel:
\begin{align*}
    P_{n+1} = \argmin_{P\in\Pc}T_P^{\pi}v^\pi_{P_n}.
\end{align*}
We continue to prove the following lemmas.
\begin{lemma}
\label{lemma:monotone_value}
The kernel update process produces monotonically decreasing value functions:
\begin{equation*}
    v_{P_n}^{\pi} \succeq v_{P_{n+1}}^{\pi},\quad \forall\,n=1,2,\cdots.
\end{equation*}
\end{lemma}
\begin{proof}   
Recall that $v_{P_n}^{\pi} = T_{P_n}^{\pi}v_{P_n}^{\pi} = R^{\pi} + \gamma P_n^{\pi}v_{P_n}^{\pi}$.
Since we have
\begin{equation*}
    P_{n+1} = \argmin_{P\in\Pc}[R^{\pi} + \gamma P^{\pi}v_{P_n}^{\pi}],
\end{equation*}
we can obtain
\begin{align*}
    & & R^{\pi} + \gamma P_n^{\pi}v_{P_n}^{\pi} &\geq  \min_{P\in\Pc}[R^{\pi} + \gamma P^{\pi}v_{P_n}^{\pi}]\\
    &\Rightarrow & v_{P_n}^{\pi} &\geq R^{\pi} + \gamma P_{n+1}^{\pi}v_{P_n}^{\pi}\\
    &\Rightarrow & (I-\gamma P_{n+1}^{\pi})v_{P_n}^{\pi} &\geq R^{\pi}\\
    &\Rightarrow & v_{P_n}^{\pi} &\geq (I-\gamma P_{n+1}^{\pi})^{-1}R^{\pi} = v_{P_{n+1}}^{\pi}. \qedhere
\end{align*}
\end{proof}

\begin{lemma}
\label{lemma:operator_monotone}
The robust bellman operators are monotonic functions, that is:
$$v \leq u\quad \Longrightarrow\quad T_{\Pc}^{\pi}v \leq T_{\Pc}^{\pi}u$$
\end{lemma}
\begin{proof}
Since $v\leq u$, and the fact that $P$ has only non-negative entries, we know that:
\begin{align*}
    & & R^\pi + \gamma P^{\pi} v &\leq R^\pi + \gamma P^{\pi}u, \quad \forall P\in\Pc\\
    &\Rightarrow & \min_{P\in \Pc}(R^\pi + \gamma P^{\pi} v) &\leq \min_{P\in \Pc}(R^\pi + \gamma P^{\pi}u)\\
    &\Rightarrow & T_{\Pc}^{\pi}v &\leq T_{\Pc}^{\pi}u
\end{align*}
\end{proof}

\kernelconv*
\begin{proof}   
We prove it by showing that:
\begin{equation*}
    \norm{v^\pi_{P_{n+1}} - v_{\Pc}^{\pi}}_{\infty} \leq \gamma \norm{v^\pi_{P_n} - v_{\Pc}^{\pi}}_{\infty}, \quad \forall\,n.
\end{equation*}
First, by optimality, we have
\begin{equation*}
    v^\pi_{P_{n+1}} - v_{\Pc}^{\pi} \geq 0.
\end{equation*}
Now we can focus only on the upper bound:
\begin{align*}
    v^\pi_{P_{n+1}} - v_{\Pc}^{\pi} &=  T_{P_{n+1}}^{\pi}v^\pi_{P_{n+1}} - T_{\Pc}^{\pi}v_{\Pc}^{\pi}\\
    & \leq T_{P_{n+1}}^{\pi}v^\pi_{P_n} - T_{\Pc}^{\pi}v_{\Pc}^{\pi} & \text{(Lemma~\ref{lemma:monotone_value} and \ref{lemma:operator_monotone})} \\
    & = \min_{P\in\Pc}T_{P}^{\pi}v^\pi_{P_n} - T_{\Pc}^{\pi}v_{\Pc}^{\pi}\\
    & = T_{\Pc}^{\pi}v^\pi_{P_n} - T_{\Pc}^{\pi}v_{\Pc}^{\pi}\\
    & \leq \gamma \norm{v^\pi_{P_n} - v_{\Pc}^{\pi}}_{\infty } & \text{($T_{\Pc}^{\pi}$ is a $\gamma$-contraction operator)}.
\end{align*}
Putting it together, we have
\begin{equation*}
    \norm{v^\pi_{P_{n+1}} - v_{\Pc}^{\pi}}_{\infty} \leq  \gamma \norm{v^\pi_{P_{n}} - v_{\Pc}^{\pi}}_{\infty}.
\end{equation*}
The desired result is proved by applying the above result iteratively.
\end{proof}


\clearpage


\section{Experiment details}
\label{app:exp-details}

\subsection{Cliff Walking}
\subsubsection{Environment description}
Cliff Walking~\footnote{\url{https://gymnasium.farama.org/environments/toy_text/cliff_walking/}} is a tabular environment from OpenAI's Gym~\citep{brockman2016openai}. Usually, this environment transition kernel is deterministic. To create a distributional uncertainty set around the nominal kernel, we We introduced stochasticity to the transition kernel. Specifically, we decreased the probability of moving in the intended direction from 1 to 0.9. Additionally, we introduced a 0.02 probability of moving in the opposite direction (e.g., Up instead of Down), and the remaining 0.08 probability is evenly distributed between moving sideways (e.g., Left or Right instead of Up). Consequently, there is a 0.04 probability of moving in any of these lateral directions.

In terms of reward - the agent receives a $-1$ penalty every time step before reaching the goal state. When it does, it encounters a reward of $100$, and each time the agent falls off the cliff, it suffers from a penalty of $-10$ and will be teleported to the initial state.

\subsubsection{Worst Environment}
\label{app:worst-env}

To determine the worst transition kernel, we computed, for each $(s, a)$ pair, the updated worst transition. This involved encouraging actions that would move the agent closer to the cliff. For instance, if the agent is positioned adjacent to the cliff and moves upward, we would try to find a transition kernel such that the probability of moving downward is maximal (while staying within the bounds of the uncertainty set).

Specifically Given $p\in\Delta_{\St}$ we want to find $q$ by solving the following optimization problem:
\begin{align*}
& \max\ q[i] \\
\mathrm{s.t.} & \sum_{s\in\St}q[s] = 1 \\
&q[s]\geq0 \forall s\in\St \\
&D_{KL}(q||p)\leq\beta
\end{align*}
where $i$ is the outcome of getting closer to the cliff.
\subsubsection{Q-Learning Hyperparameters}
To obtain the optimal policy we used Value iteration ~\citep{Puterman1994MarkovDP}, using the access to the true kernel (either nominal or worst).
To learn the robust policy we used a simple Q-learning algorithm ~\citep{watkins1992q} using samples from the nominal kernel (while allowing multiple samples for any timesteps). The detailed configurations for hyperparameters used in this experiment are summarized in Table~\ref{tab:Cliff_params}.

\begin{table}[ht]
    \centering
    \caption{Hyperparameters used in Cliff Walking environment}
    \vspace{5pt}
    \begin{tabular}{cc}
    \toprule
    \textsc{Parameter} & \textsc{value}\\
    \midrule
    \multicolumn{2}{l}{\textit{Q-Learning Hyperparameters}} \\
    \midrule
    Number of episodes  & $20000$ \\
    Learning rate & $0.01$ \\
    Discount factor ($\gamma$) & $0.8$ \\
    Exploration factor $\epsilon$ & $0.2$ \\
    Value Error stopping condition & $1e^{-6}$ \\
    \midrule
    \multicolumn{2}{l}{\textit{Uncertainty set parameters}} \\
    \midrule
    $\beta$ & $0.4$ \\
    \midrule
    \multicolumn{2}{l}{\textit{EWoK Hyperparameters}} \\
    \midrule
    Number of samples ($N$) & $5$ \\
    $\kappa$ & $0.4$ \\
    \bottomrule
    \end{tabular}
    \label{tab:Cliff_params}
\end{table}

\subsection{Environments}
\label{app:env}
\subsubsection{Classic control tasks}

Cartpole~\footnote{\url{https://gymnasium.farama.org/environments/classic_control/cart_pole/}} is one of the classic control tasks in OpenAI Gym~\citep{brockman2016openai}.
The task is to balance a pendulum on a moving cart, by moving the cart either left or right.
The state consists of the location and velocity of the cart, as well as the angle and angular velocity of the pendulum.
To make the transition dynamic stochastic, we add Gaussian noises to the cart position.
The detailed configurations for the nominal values and the perturbation ranges are summarized in Table~\ref{tab:Cartpole_params}.

\begin{table}[ht]
    \centering
    \caption{Perturbation configurations for Cartpole environment.}
    \vspace{5pt}
    \begin{tabular}{cccc}
    \toprule
    & \textsc{Parameter} & \textsc{Nominal value} & \textsc{Perturbation range} \\
    \midrule
    \makecell{\textsc{Noise}\\ \textsc{pertubration}} & Cart position noise (std) & $0.01$ & $[0,0.1]$ \\
    \midrule
    \multirow{4}{*}{\makecell{\textsc{Env. param.}\\ \textsc{pertubration}}} & Pole mass & $0.1$ & $[0.15,3.0]$ \\
    & Pole length  & $0.5$ & $[0.25,5.0]$ \\
    & Cart mass  & $1$ & $[0.25,5.0]$ \\
    & Gravity  & $9.8$ & $[0.1,30]$ \\
    \bottomrule
    \end{tabular}
    \label{tab:Cartpole_params}
\end{table}

\subsubsection{DeepMind Control Suite}

The DeepMind Control Suite~\citep{tunyasuvunakool2020} is a set of continuous control tasks powered by the MuJoCo physics engine~\citep{todorov2012mujoco}.
It is widely used to benchmark reinforcement learning agents.
As mentioned in the main text, we consider 3 tasks in our paper: \texttt{walker-stand}, \texttt{walker-walk}, \texttt{walker-run}.
For those tasks, the observations are 24-dimensional vectors and the actions are 6-dimensional vectors.
For noise perturbation, we fix the standard deviation of the Gaussian noise to 0.2.
The nominal value and the perturbation range are summarized in Table~\ref{tab:Walker_params}.

\begin{table}[ht]
    \centering
    \caption{Perturbation configurations for DeepMind Control Suite tasks.}
    \vspace{5pt}
    \begin{tabular}{cccc}
    \toprule
    & \textsc{Parameter} & \textsc{Nominal value} & \textsc{Perturbation range} \\
    \midrule
    \makecell{\textsc{Noise} \\ \textsc{pertubration}} & action noise (mean) & $0.0$ & $[-0.3,0.3]$ \\
    \midrule
     \multirow{3}{*}{\makecell{\textsc{Env. param.} \\ \textsc{pertubration}}} & thigh length & $0.225$ & $[0.1, 0.5]$ \\
     & torso length & $0.3$ & $[0.1,0.7]$ \\
     & joint damping & $0.1$ & $[0.1,10]$ \\
    \bottomrule
    \end{tabular}
    \label{tab:Walker_params}
\end{table}

\subsection{Training and evaluation}

For our method and the baseline, we first train the agent under the nominal environment, and then for each perturbed environment during testing, we calculate the average reward from 30 episodes.
We repeat this process with 40 random seeds for DDQN experiments and 20 random seeds for PPO experiments in the classic control environments.
For DeepMind Control environments we used 10 different seeds for SAC experiments and 5 different seeds for TD3 experiments.
Following the recommended practice in~\citep{agarwal2021deep}, we report the Interquartile Mean (IQM) and the $95\%$ stratified bootstrap confidence intervals (CIs), using 
The IQM metric is measured by discarding the top and bottom $25\%$ of the results, and averaging across the remaining middle $50\%$.
IQM has the benefit of being more robust to outliers than a regular mean, and being a better estimator of the overall performance than the median.
We use the rliable library\footnote{\url{https://github.com/google-research/rliable}} to calculate IQM and CIs.



As mentioned earlier, we use Double-DQN~\citep{ddqn} and PPO~\citep{schulman2017proximal} as the vanilla non-robust RL algorithms for environments with discrete action spaces.
Specifically, we follow the implementation in Stable-Baselines3~\citep{stable-baselines3}.
For the Cartpole environment, we use the hyperparameters suggested in RL Baselines3 Zoo~\citep{rl-zoo3}. The detailed configurations are summarized in Table~\ref{tab:DQN_params} and Table~\ref{tab:PPO_params}.

\begin{table}[ht]
    \centering
    \caption{Hyperparameters for training cartpole with DDQN used in the experiments.}
    \vspace{5pt}
    \begin{tabular}{ccc}
    \toprule
    \textsc{Parameter} & \textsc{Value} \\
     \midrule
     batch size   & 64   \\
     buffer size&   100000\\
     exploration final epsilon&0.04\\
     exploration fraction&0.16 \\
     gamma&   0.99  \\
     gradient steps& 128   \\
     learning rate& 0.0023  \\
     learning starts& 1000 \\
     target update interval& 10  \\
     train frequency& 256 \\
     total time-steps& 50000\\
    \bottomrule
    \end{tabular}
    \label{tab:DQN_params}
\end{table}

\begin{table}[ht]
    \centering
    \caption{Hyperparameters for training cartpole with PPO used in the experiments.}
    \vspace{5pt}
    \begin{tabular}{ccc}
    \toprule
    \textsc{Parameter} & \textsc{Value} \\
     \midrule
    number of parallel environments   & 1   \\
     batch size   & 32   \\
     Number of steps  & 32   \\
     gae lambda   & 0.98   \\
     gamma&   0.98  \\ 
     Number of epochs&   20  \\ 
     entropy coefficient&   0 \\ 
     clip range&   0.2 \\ 
     learning rate& 0.001  \\
     total time-steps& 100000\\
    \bottomrule
    \end{tabular}
    \label{tab:PPO_params}
\end{table}

For environments with continuous action spaces, we choose the SAC algorithm~\citep{haarnoja2018soft} and TD3~\citep{fujimoto2018addressing} as the vanilla non-robust RL algorithms.
For SAC, we follow the implementations and hyperparameters choices in~\citep{pytorch_sac}.
Both the actor and critic use a two-layer MLP neural network with 1024 hidden units per layer. Table~\ref{tab:SAC_params} lists the hyperparameters.
For TD3, we follow the implementations and hyperparameters choices in~\citep{stable-baselines3}.
Table~\ref{tab:TD3_params} lists the hyperparameters.

\begin{table}[ht]
    \centering
    \caption{Hyperparameters for SAC used in the experiments.}
    \vspace{5pt}
    \begin{tabular}{cc}
    \toprule
    \textsc{Parameter} & \textsc{Value} \\
    \midrule
    Total steps & 1e6 \\
    Warmup steps & 5000 \\
    Replay size & 1e6 \\
    Batch size & 1024 \\
    Discount factor $\gamma$ & 0.99 \\
    Optimizer & Adam~\citep{kingma14adam} \\
    Learning rate & 1e-4 \\
    Target smoothing coefficient & 0.005 \\
    Target update interval & 2 \\
    Actor update interval & 1 \\
    Initial temperature & 0.1 \\
    Learnable temperature & Yes \\
    \bottomrule
    \end{tabular}
    \label{tab:SAC_params}
\end{table}

\begin{table}[ht]
    \centering
    \caption{Hyperparameters for TD3 used in the experiments.}
    \vspace{5pt}
    \begin{tabular}{cc}
    \toprule
    \textsc{Parameter} & \textsc{Value} \\
    \midrule
    Total steps & 1e6 \\
    Warmup steps & 5000 \\
    Replay size & 1e6 \\
    Batch size & 1024 \\
    Discount factor $\gamma$ & 0.99 \\
    Optimizer & Adam~\citep{kingma14adam} \\
    Learning rate & 1e-4 \\
    Target smoothing coefficient & 0.005 \\
    Target update interval & 2 \\
    Actor update interval & 2 \\
    Target policy noise std & 0.2 \\
    Target policy noise clip & 0.5 \\
    \bottomrule
    \end{tabular}
    \label{tab:TD3_params}
\end{table}

The configurations for the sample size $N$ and the robustness parameter $\kappa$ used in our experiments are summarized in Table~\ref{tab:AKA_params}.

\begin{table}[ht]
    \centering
    \caption{Hyperparameters specific to our method used in the experiments.}
    \vspace{5pt}
    \begin{tabular}{cccc}
    \toprule
    & \textsc{Environment} & \textsc{sample size $N$} & \textsc{robustness parameter $\kappa$} \\
    \midrule
      \multirow{2}{*}{\makecell{\textsc{classic contro}}} & Cartpole (DDQN)& $15$ & $0.1$ \\
     & Cartpole (PPO) & $15$ & $0.2$ \\
      \midrule
      \multirow{3}{*}{\makecell{\textsc{DeepMind} \\ \textsc{Control} \\ \textsc{Suite}}} & walker-stand (SAC \& TD3)& $10$ & $0.2$ \\
      & walker-walk (SAC \& TD3)& $10$ & $0.2$ \\
      & walker-run (SAC \& TD3)& $10$ & $0.2$ \\
  \bottomrule
    \end{tabular}
    \label{tab:AKA_params}
\end{table}

\subsection{Computational resources and costs}

We used the following resources in our experiments:
\begin{itemize}
    \item \textbf{CPU:} AMD EPYC 7742 64-Core Processor
    \item \textbf{GPU:} NVIDIA GeForce RTX 2080 Ti
\end{itemize}
Table~\ref{tab:time cost} lists the training time.

\begin{table}[ht]
    \centering
    \caption{Training time per run of our experiments on a single GPU.}
    \vspace{5pt}
    \begin{tabular}{cccc}
    \toprule
    & \textsc{Environment} & \textsc{Baseline} & \textsc{EWoK(Ours)} \\
    \midrule
      \textsc{classic control (DDQN)} & Cartpole & $\sim$ 4 minutes & $\sim$ 5 minutes \\
      \textsc{classic control (PPO)} & Cartpole & $\sim$ 60 minutes & $\sim$ 60 minutes \\
      \midrule
      \textsc{DM Control (SAC)}
      & all tasks & $\sim$ 5 hours & $\sim$ 6 hours \\ 
      \textsc{DM Control (TD3)}
      & all tasks & $\sim$ 6 hours & $\sim$ 7 hours \\
  \bottomrule
    \end{tabular}
    \label{tab:time cost}
\end{table}


\clearpage
\section{Additional results}
\label{app:all-results}

In section~\ref{subsection:env_params} we show the result of testing EWoK on different parameters perturbation. Here we show the results of all of the experimented parameters we have perturbed during test time. The results for SAC algorithm can be seen in Figure~\ref{fig:dm_sac} and the results for TD3 algorithm can be seen in Figure~\ref{fig:dm_td3}.

In section \ref{subsection:ablation}, we show the relative performance for the ablation study on parameter $\kappa$ and $N$ for two environments only. In Figures~\ref{fig:appe-ablation-kappa} and ~\ref{fig:appe-ablation-N} we depict the full results for all of the 3 environments. In Figures~\ref{fig:appe-ablation-kappa-absolute} and ~\ref{fig:appe-ablation-N-absolute} we also include the absolute results (with the reference baseline algorithm).

\begin{figure}[ht]
\centering
\includegraphics[width=0.75\linewidth]{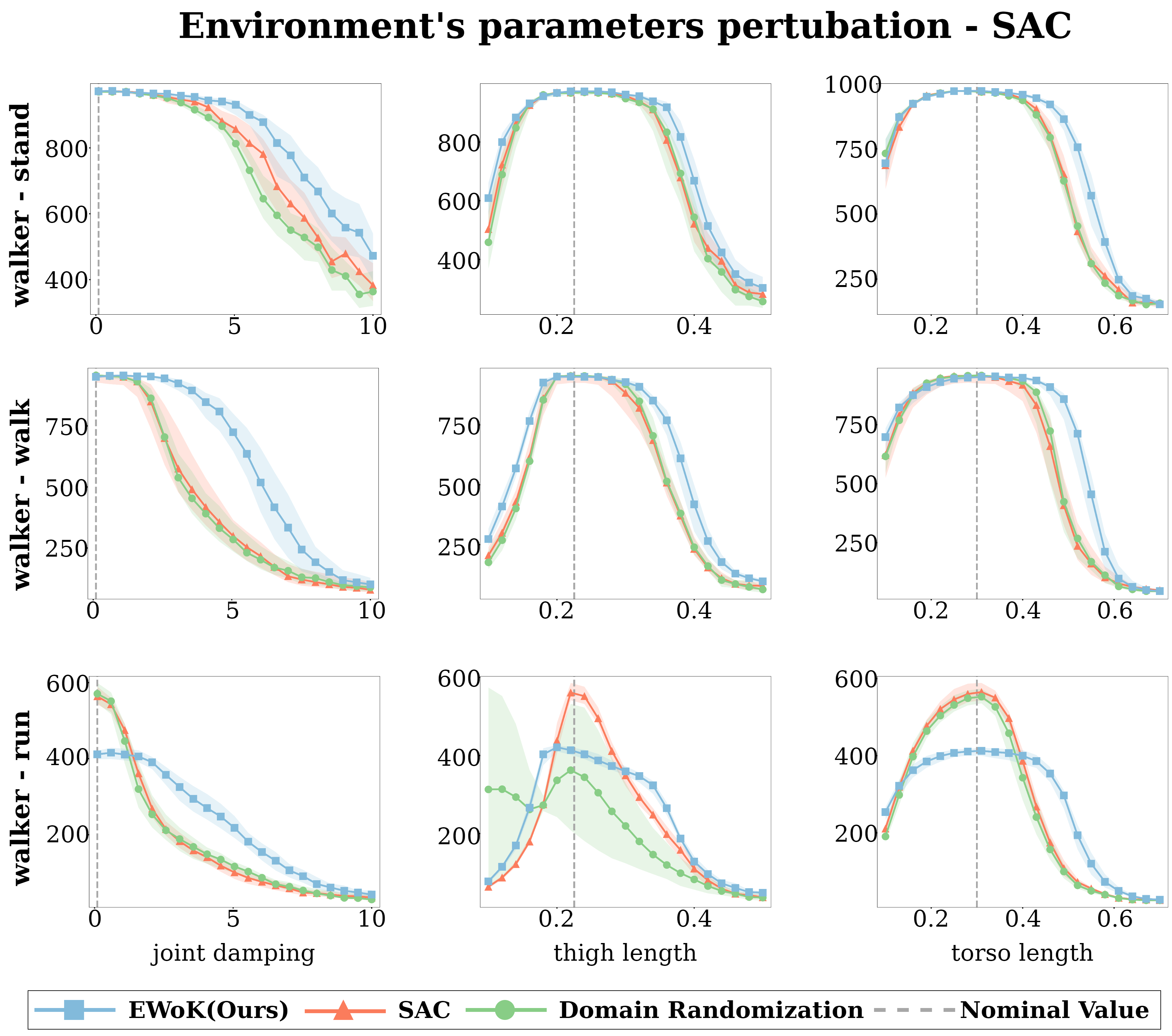}
\caption{Evaluation results on DeepMind Control environments with environment's parameters perturbations for SAC algorithm.}
\label{fig:dm_sac}
\end{figure}

\begin{figure}[ht]
\centering
\includegraphics[width=0.75\linewidth]{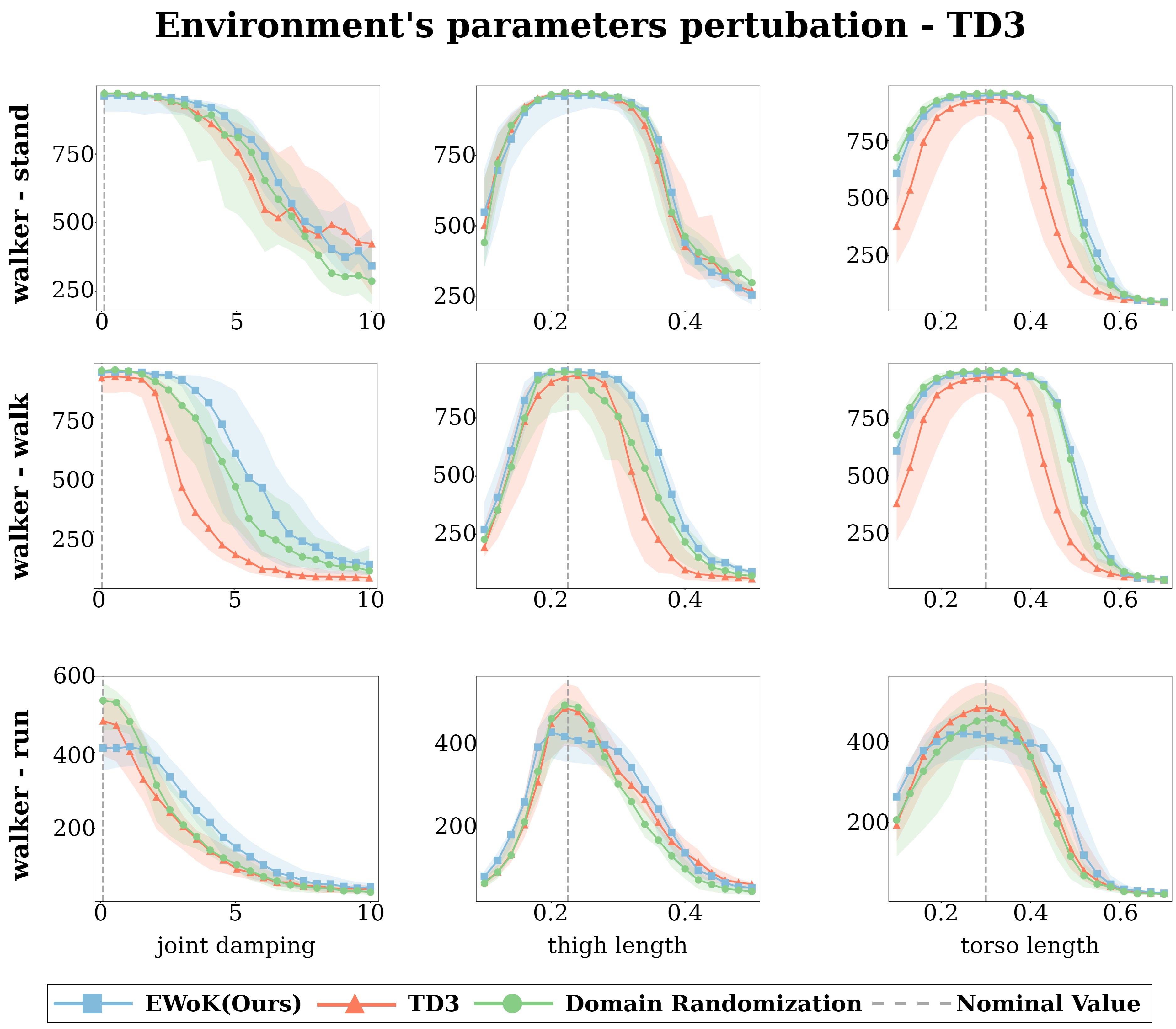}
\caption{Evaluation results on DeepMind Control environments with environment's parameters perturbations for TD3 algorithm.}
\label{fig:dm_td3}
\end{figure}


\begin{figure}[ht]
\centering
\includegraphics[width=0.75\linewidth]{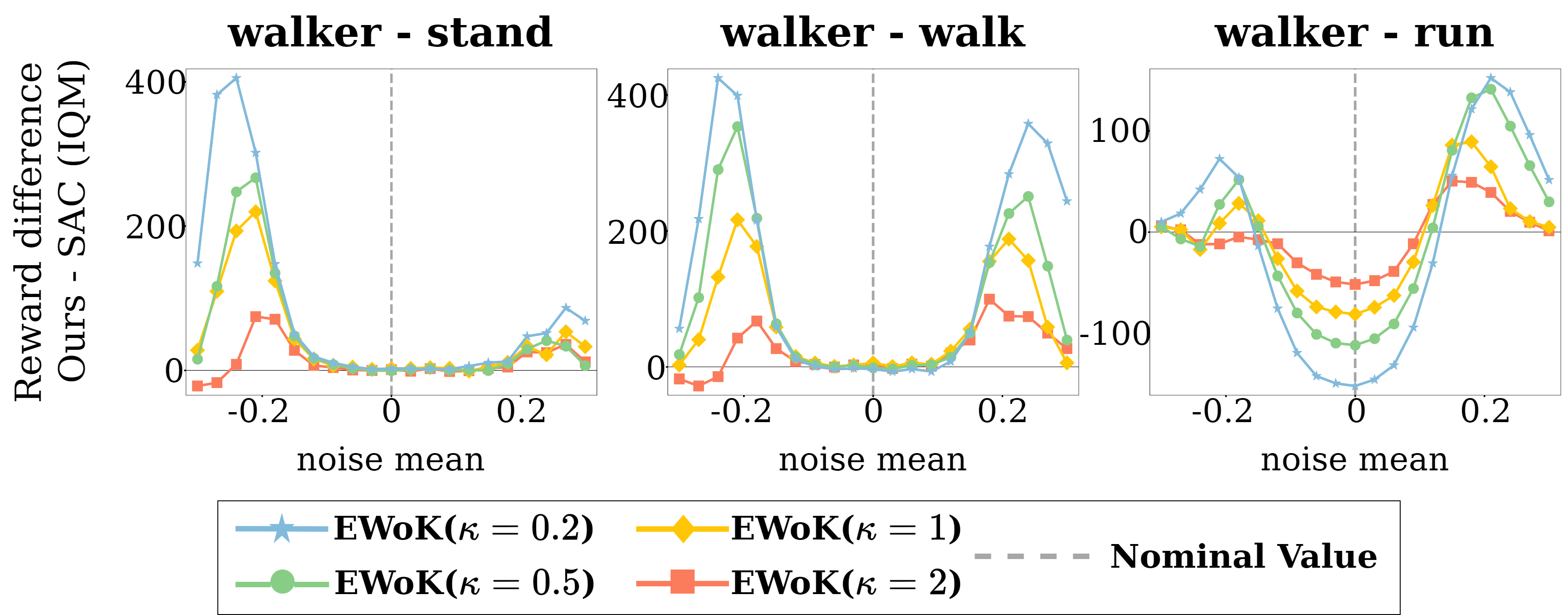}
\caption{Evaluation results on DeepMind Control tasks with noise perturbations for different $\kappa$.}
\label{fig:appe-ablation-kappa}
\end{figure}

\begin{figure}[ht]
\centering
\includegraphics[width=0.75\linewidth]{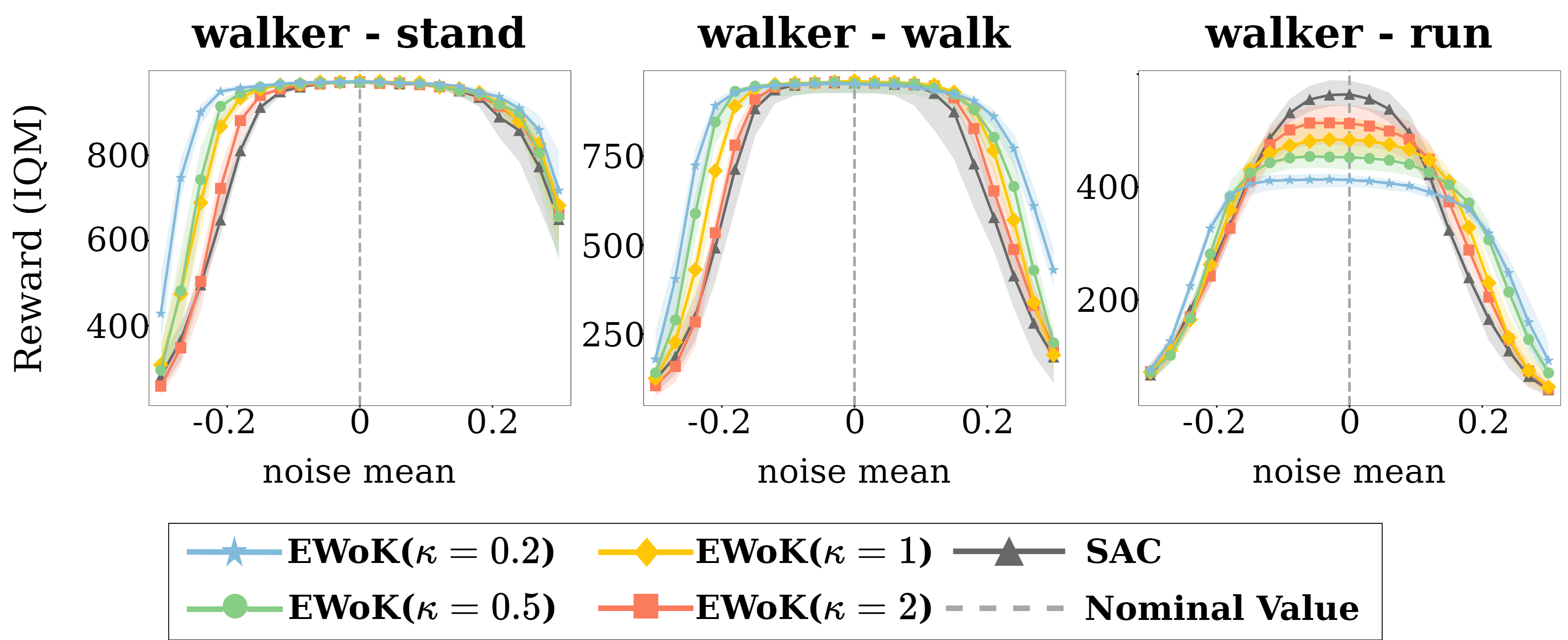}
\caption{Evaluation results (absolute) on DeepMind Control tasks with noise perturbations for different $\kappa$.}
\label{fig:appe-ablation-kappa-absolute}
\end{figure}

\begin{figure}[ht]
\centering
\includegraphics[width=0.75\linewidth]{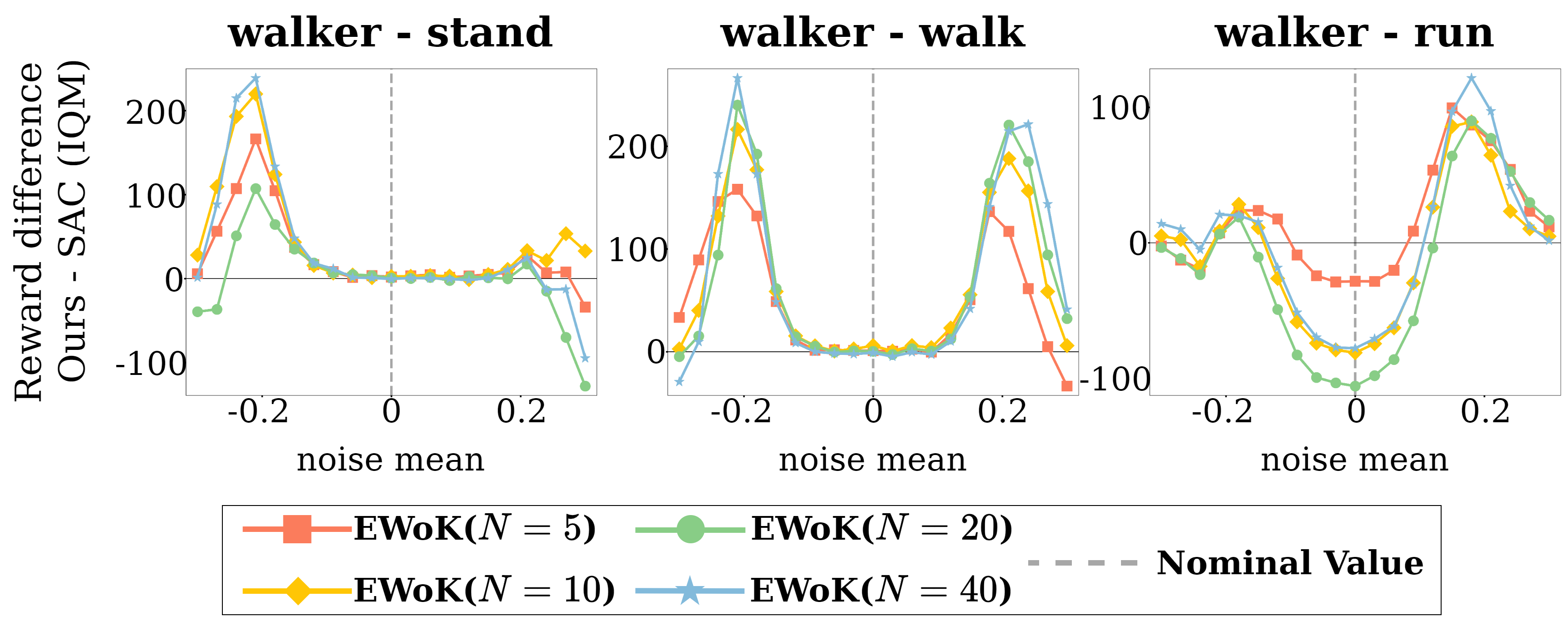}
\caption{Evaluation results on DeepMind Control tasks with noise perturbations for different $N$.}
\label{fig:appe-ablation-N}
\end{figure}

\begin{figure}[ht]
\centering
\includegraphics[width=0.75\linewidth]{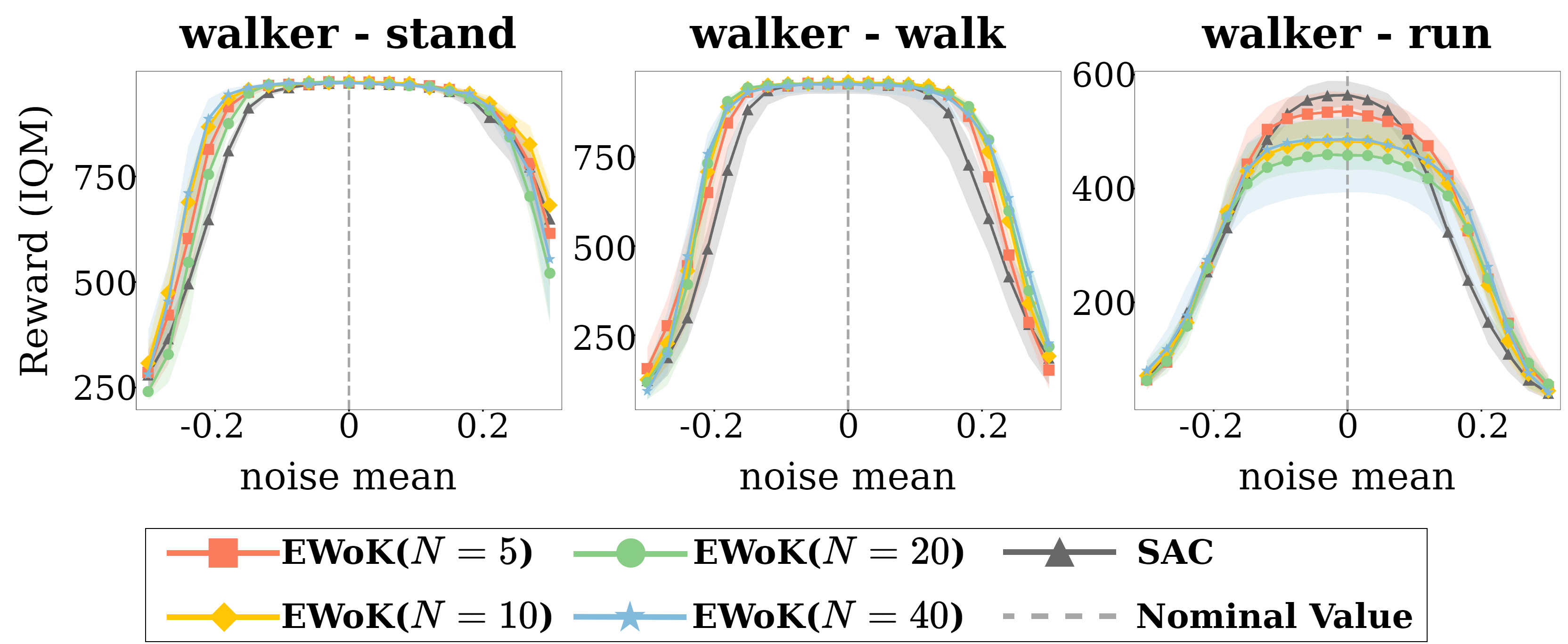}
\caption{Evaluation results (absolute) on DeepMind Control tasks with noise perturbations for different $N$.}
\label{fig:appe-ablation-N-absolute}
\end{figure}



\end{document}